\documentclass{arxiv}
\usepackage[foot]{amsaddr}
\usepackage{xcolor}
\usepackage{todonotes}
\usepackage[hypcap=false]{caption}
\usepackage{subcaption, enumitem}
\usepackage[mathlines]{lineno}
\usepackage{cite}
\usepackage{wrapfig}

\hypersetup{
	pdfauthor={C. Bayer, P. Friz, N. Tapia},
	pdfkeywords={Neural networks, stability, rough paths}
}

\title{Stability of Deep Neural Networks via discrete rough paths}
\author[C. Bayer]{Christian Bayer$^1$, Peter Friz$^{1,2}$ and Nikolas Tapia$^{1,2}$}
\address{$^1$Weierstrass Institute}
\address{$^2$TU Berlin}
\email{christian.bayer@wias-berlin.de, friz@math.tu-berlin.de, tapia@wias-berlin.de}

\mleftright
\let\letter\mathsf

\def\b{\mathrm{b}}

\def\bbW{\mathbb{W}}
\def\bW{\mathbf{W}}

\def\bw{\mathbf{w}}

\def\bx{\mathbf{x}}
\def\by{\mathbf{y}}
\def\bz{\mathbf{z}}

\def\W[#1]{[\letter{#1}]}
\def\t|#1|{%
  \mathopen{\lvert\mkern-3mu\lvert\mkern-3mu\lvert}%
    #1%
  \mathclose{\rvert\mkern-3mu\rvert\mkern-3mu\rvert}%
}

\newcommand{\christian}[1]{#1}

\begin{document}
\maketitle
\begin{abstract}
Using rough path techniques, we provide \emph{a priori} estimates for the output of Deep Residual Neural Networks in terms of both the input data and the (trained) network weights. As trained network weights are typically very rough when seen as functions of the layer, we propose to derive stability bounds in terms of the total \(p\)-variation of trained weights for any \(p\in[1,3]\). Unlike the $C^1$-theory underlying the neural ODE literature, our estimates remain bounded even in the limiting case of weights behaving like Brownian motions, as suggested in [Cohen, Cont, Rossier, Xu: ``Scaling properties of deep residual networks'', arXiv, 2021]. Mathematically, we interpret residual neural network as solutions to (rough) difference equations, and analyze them based on recent results of discrete time signatures and rough path theory.
\end{abstract}

\section{Introduction}
\begin{wrapfigure}{r}{0.4\textwidth}
	\begin{center}
	\includegraphics[width=0.2\textwidth]{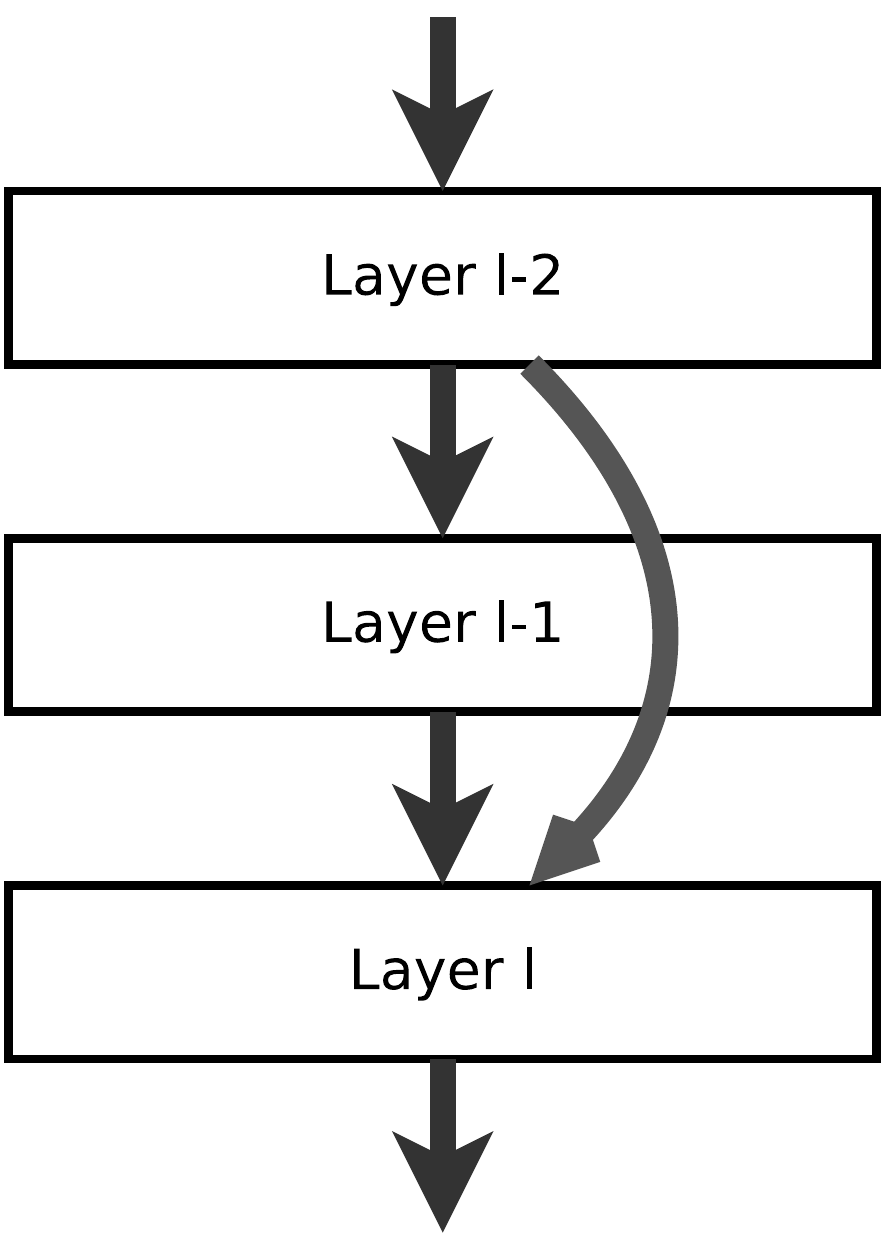}
\end{center}
\caption{Single block of the ResNet architecture}
\label{fig:rnn}
\end{wrapfigure}

Since their introduction in 2016 \cite{HZRS16}, Residual Neural Networks (ResNets) have gained a vast amount of popularity as a preferred network
architecture for Machine Learning applications.  The general principle is that this architecture allows for deeper networks since it models only the
residual change of the features at the output of each layer.  This is achieved by introducing ``skip connections'' which -- at some steps -- adjust
the output of a layer by adding an earlier layer's output (see \Cref{fig:rnn}).
These ``blocks'', formed by a sequence of layers connected by an identity mapping, are then stacked on top of each other in order to build the
network.

The authors of the previously cited paper argue that this helps precondition the optimization solvers so that increasing the network depth does
not result in severe numerical instabilities and performance degradation, as is observed in plain Neural Networks.  In particular, this approach
allows them to successfully train a Deep Neural Network with hundreds and even thousands \cite{HZR+2016} of layers.

In a plain Neural Network, the input vector \(\by_{i+1}\) of the \((i+1)\)-th hidden layer is given by an application of the weights and the
activation function to the input of the previous hidden layer. In symbols
\[
\by_{i+1}=\sigma(\theta_i\by_i)
\]
where
\(\sigma\colon\R^{d_{i+1}}\to\R^{d_{i+1}}\) and \(\theta_i\) is a \(d_{i+1}\times d_i\) matrix.
In the ResNet approach, this is modified so that the output to the next hidden layer is given as the sum of the \emph{input} to the previous layer,
plus the previous operations; that is, \begin{equation} \by_{i+1}=\by_i+\sigma(\theta_i\by_i).  \label{eqn:rnn.upd} \end{equation} Here, it is assumed
that the width of all layers is constant, but the approach can easily be adapted to the more familiar setting of varying widths by applying an
appropriate projection to right-hand side of the last equation.

\begin{remark}
We simplify notation by leaving out the bias term in the update rule
\eqref{eqn:rnn.upd}. The usual update rule \begin{equation*} \by_{i+1} = \sigma(\theta_i
\by_i + b_i) \end{equation*} can be reproduced in the form \eqref{eqn:rnn.upd} above
by adding a column of consisting of ones to $\by_i$ and an appropriate restriction on
$\theta_i$ to map that column to another column of ones -- in the appropriate dimension.
\end{remark}

\begin{remark}
In this work, we assume that the architecture follows the update
\eqref{eqn:rnn.upd} at each layer. In the engineering practice, usually a few layers are
skipped over. i.e. the true update may look as follows:
\begin{equation*}
  \widetilde{\by}_i = \sigma(\widetilde{\theta}_i \by_i), \quad \by_{i+1} =
  \by_i + \sigma(\theta_i \widetilde{\by}_i),
\end{equation*}
skipping over one layer in the process.
\end{remark}

It has been argued by several authors \cite{W2017,HR18,HRH+2018} that the update in \cref{eqn:rnn.upd} can
be seen as a step of the Euler scheme for a \emph{controlled ODE} of the form
\begin{equation}
  \dot{\by}(t)=\sigma(\theta(t)\by(t)),\quad \by(0)=\by_0.
  \label{eqn:rnn.ode}
\end{equation}
Then, knowledge of stability and convergence of numerical schemes for such systems
can be used to derive corresponding results for ResNets, especially since one expects that the
behavior of the output layer of the network under consideration will follow closely that
of the continuous-time solution of \cref{eqn:rnn.ode} (that is, in the limit of infinite depth) for very deep architectures.

In this work, we go back one step and consider the situation of ResNets with many, but finitely many layers.
Specifically, we consider finite difference equations of the form
\begin{equation}
  \bx_{k+1}=\bx_k+\sum_{\mu=1}^df_\mu(\bx_k)(\bw^\mu_{k+1}-\bw^\mu_k),\quad\bx_0=\xi\in\R^m.
  \label{eq:creintro}
\end{equation}
Here, in the simplest case of constant dimension $d$, $\bx_k$ denotes the vector of nodes at layer $k$ (corresponding to $\by_k$ above),
	and the increment matrix $(\bw^\mu_{k+1} - \bw^\mu_k)_{\mu = 1}^d$ corresponds to the matrix $\theta_k \in \R^{d \times d}$. Finally, the vector
	fields $f_\mu:\R^d \to \R^d$ take care of the matrix-vector-multiplication as well as of the non-linear activation function $\sigma$. We assume that
	the number of nodes $d$ is constant over all the layers. For a more general and detailed view of the setting, we refer to
	Appendix~\ref{sec:comp-arch}.

	Already from this very cursory look, the reader may notice an apparent difference between~\eqref{eqn:rnn.upd} and~\eqref{eq:creintro}: in
	the former formulation the nonlinearity is applied after the matrix multiplication, whereas the order of operations is reversed in our finite
	difference equation. However, when we consider a deep network, this difference essentially only effects the very first layer of the network, which is
	hit by the nonlinearity in~\eqref{eq:creintro} before any affine transform is applied. All other layers are treated exactly the same way by both
	architectures -- assuming that the non-linearity is not applied to the output layer, as is customarily the case.

Hence, it does not come as a surprise that both formulations are essentially equivalent, as also pointed out in \cite{MSK+2021}. We refer to Appendix~\ref{sec:comp-arch} for a detailed analysis in our setting.

As seen in \eqref{eq:creintro}, instead of using continuous-time techniques, our approach consists of analyzing the evolution of the sequence
\((\bx_0,\dotsc,\bx_N)\), where \(N\) is the depth of the network, obtained by iteration of \cref{eqn:rnn.upd} directly at the discrete level. Seeing~\eqref{eq:creintro} as discretization of an ODE amounts to assuming that the weight sequence comes from a $C^1$-path of finite variation. There are conceptual and numerical reasons, discussed below, that suggest a less restrictive view, formulated in the so-called \emph{$p$-variation scale}. Recall that the $p$-variation seminorm of a sequence $(\bw_0, \ldots, \bw_N)$ is given by
\[
  \|\bw\|_{p;[0,N]} \coloneqq \left( \max_{s\in\mathcal S_{0,N}} \sum_{j=0}^{\#s}|\bw_{s_{j+1}}-\bw_{s_j}|^p \right)^{1/p} \]
    where the maximum is taken over the set \(\mathcal S_{0,N}\) of all increasing subsequences \[ s=(s_0=0,s_1,\dotsc,s_m,s_{m+1}=N) \]
    of \(\{0,\dotsc,N\}\) and we have set \(\#s=m\) for such a sequence. We use analytic techniques borrowed from rough paths theory and the algebraic framework developed in \cite{DET2020} to contributes to our understanding of stability properties of deep neural networks. We have

\begin{theorem}
	Suppose \(\bx,\tilde{\bx}\) are two solutions to \cref{eq:cre} with initial conditions \(\xi,\tilde{\xi}\) and driven by \(\bw,\tilde{\bw}\)
	respectively.
\begin{itemize}  
    \item Let \(1\le p<2\) and \(f_1,\dotsc,f_d\in\cC^2_\b\). Then 
    \[
			\sup_{k=0,\dotsc,N}|\bx_k-\tilde{\bx}_k|\le 2c^{1/p}_{p,N}e^{c_{p,N}\|f\|_{\cC^2_\b}^p\|\tilde\bw\|_{p;[0,N]}^p}(|\xi-\tilde{\xi}|+\|f\|_{\cC^2_\b}\|\bw-\tilde\bw\|_{p;[0,N]})
    \]
		holds, where $c_{p,N}$ is explicitly given in \Cref{thm:main.young} below.
     \item 
    Let \(2\le p<3\) and 
    \(f_1,\dotsc,f_d\in\cC^3_\b\). Then
    \[
			\sup_{k=0,\dotsc,N}|\bx_k-\tilde{\bx}_k|\le
			2(c_{p,N}')^{1/p}e^{c_{p,N}\|f\|_{\cC^3_\b}^p\left(\t|\bW|^p_{p;[0,N]}+\t|\tilde\bW|_{p;[0,N]}^p\right)}(|\xi-\tilde{\xi}|+\|f\|_{\cC^3_\b}\rho_p(\bW,\tilde\bW))
    \]
		holds, where $c_{p,N}$ is again explicitly given in \Cref{thm:main.resnet} below.
       \end{itemize}
	\label{thm:main.resnetintro}
\end{theorem}

The symbol \(\bW\) denotes the discrete signature lift of the weight sequence \(\bw\) appearing in \cref{eq:creintro} (see \Cref{sse:lift}) and
\(\t|\cdot|_{p}\) is an appropriately defined norm on the spaces of lifts (\Cref{thm:main.resnet}).  This inequality holds \emph{uniformly} over input
data.  In practice, the weight matrices are randomly initialized with random i.i.d. values so typically the trained weights are also random.  Our
estimates hold pathwise, in the sense that the depend only on a single initialization of the weight matrices.
Typically the size of the constants \(c_{p,N},c_{p,N}'\) appearing in \Cref{thm:main.resnetintro} can be very large, but they remain uniformly bounded
as \(N\to\infty\) for all fixed \(p\in[1,3)\).
The fact that these constant can take on large values is also a consequence of the pathwise nature of our estimates, in the sense that they control
the worst-case behavior of the network. We expect that under some assumptions on the distribution on the weights, some tighter control can be obtained
for the average-case behavior. In the continuous-time setting, the corresponding analysis has been performed by e.g. Cass, Litterer and Lyons \cite{CLL2013}.

To see how our a priori estimate compares to what the $C^1$ theory would imply, we ran a
simple numerical experiment\footnote{Code available on
GitHub, at \url{https://github.com/ntapiam/resnets}.}, by first training a ResNet128 using the MNIST dataset and then computing 
the \(p\)-variation of the weights and their lift (\Cref{fig:pvars}).
The jump observed at \(p=2\) is produced by switching from the standard \(p\)-variation norm \(\|\cdot\|_p\) to the augmented \(p\)-variation norm
\(\t|\cdot|_p\). \christian{To put \Cref{fig:pvars} into context, note that the classical $C^1$ analysis estimate corresponds to the case $p = 1$ in our theory. (\Cref{fig:resnet-weights} shows one entry of the matrices $\bw_k$ plotted against the time index $k$ as well as the same entry of the differences $\bw_{k+1} - \bw_k$. Specifically, we plot the entry with indices $(0,0)$. Similarly, \Cref{fig:resnet.bound} shows the value of two entries of the vector of nodes $\bx_k$ plotted against the layer $k$. In this case, we chose the entries with indices $0$ and $32$, respectively. The choices of particular entries are arbitrary.}

\begin{figure}[ht]
  \centering
  \begin{subfigure}[b]{0.48\textwidth}
    \centering
    \includegraphics[width=\textwidth]{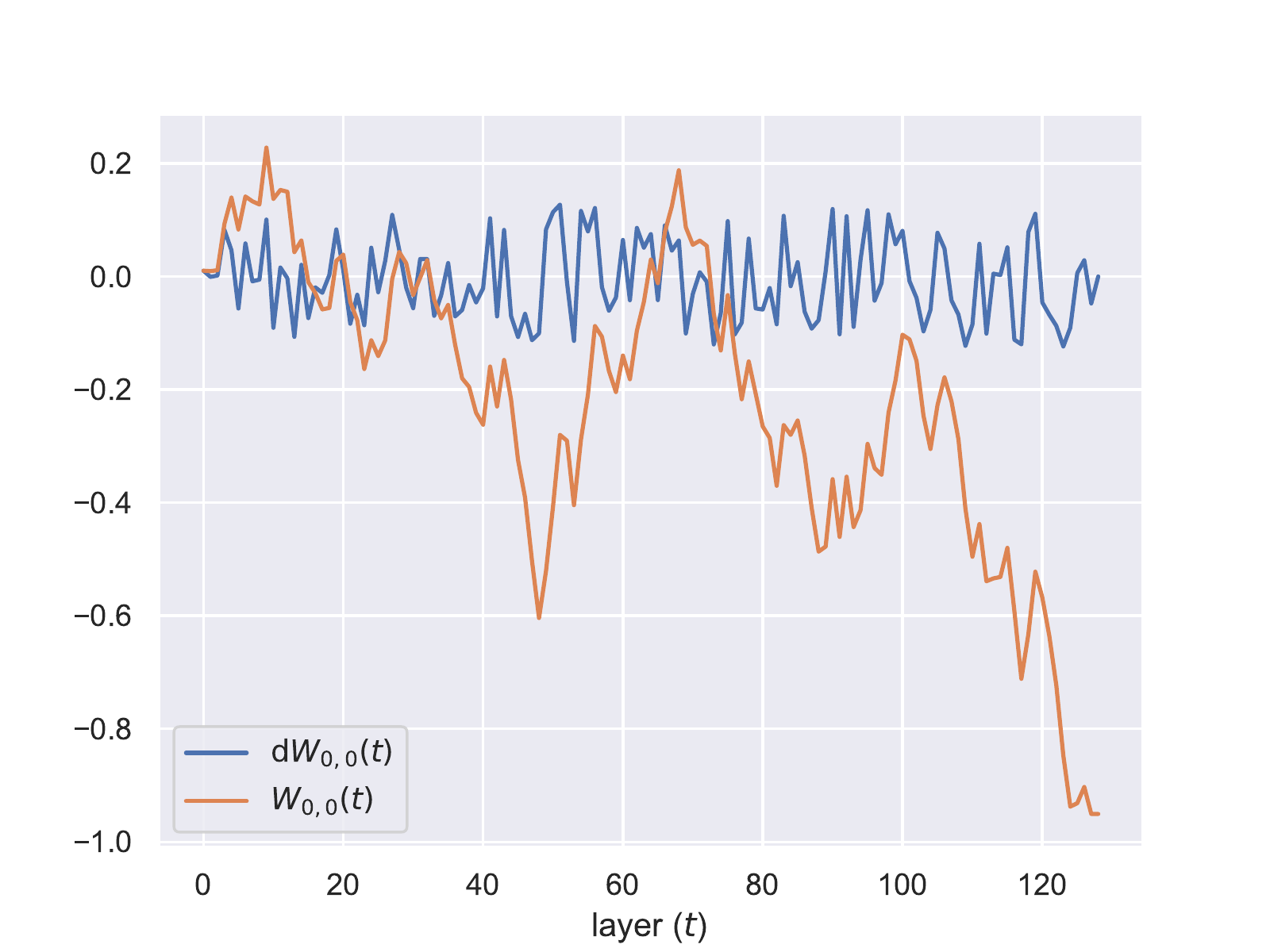}
    \caption{Evolution of selected weights.}
    \label{fig:resnet-weights}
  \end{subfigure}
  \hfill
  \begin{subfigure}[b]{0.48\textwidth}
    \centering
    \includegraphics[width=\textwidth]{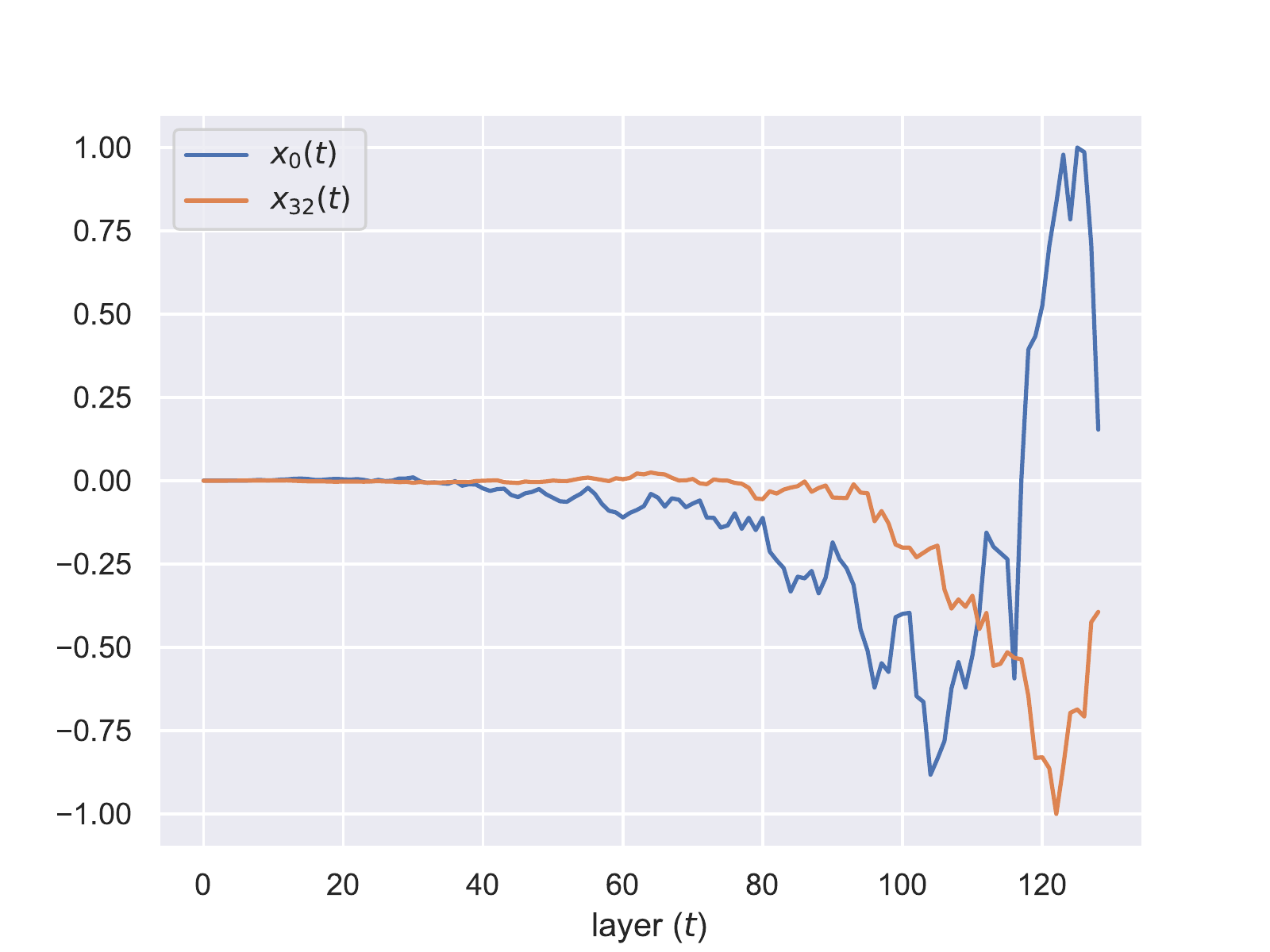}
    \caption{Evolution of selected features, rescaled to lie in the interval \([-1,1]\).}
    \label{fig:resnet.bound}
  \end{subfigure}
  \begin{subfigure}[b]{0.5\textwidth}
    \centering
    \includegraphics[width=\textwidth]{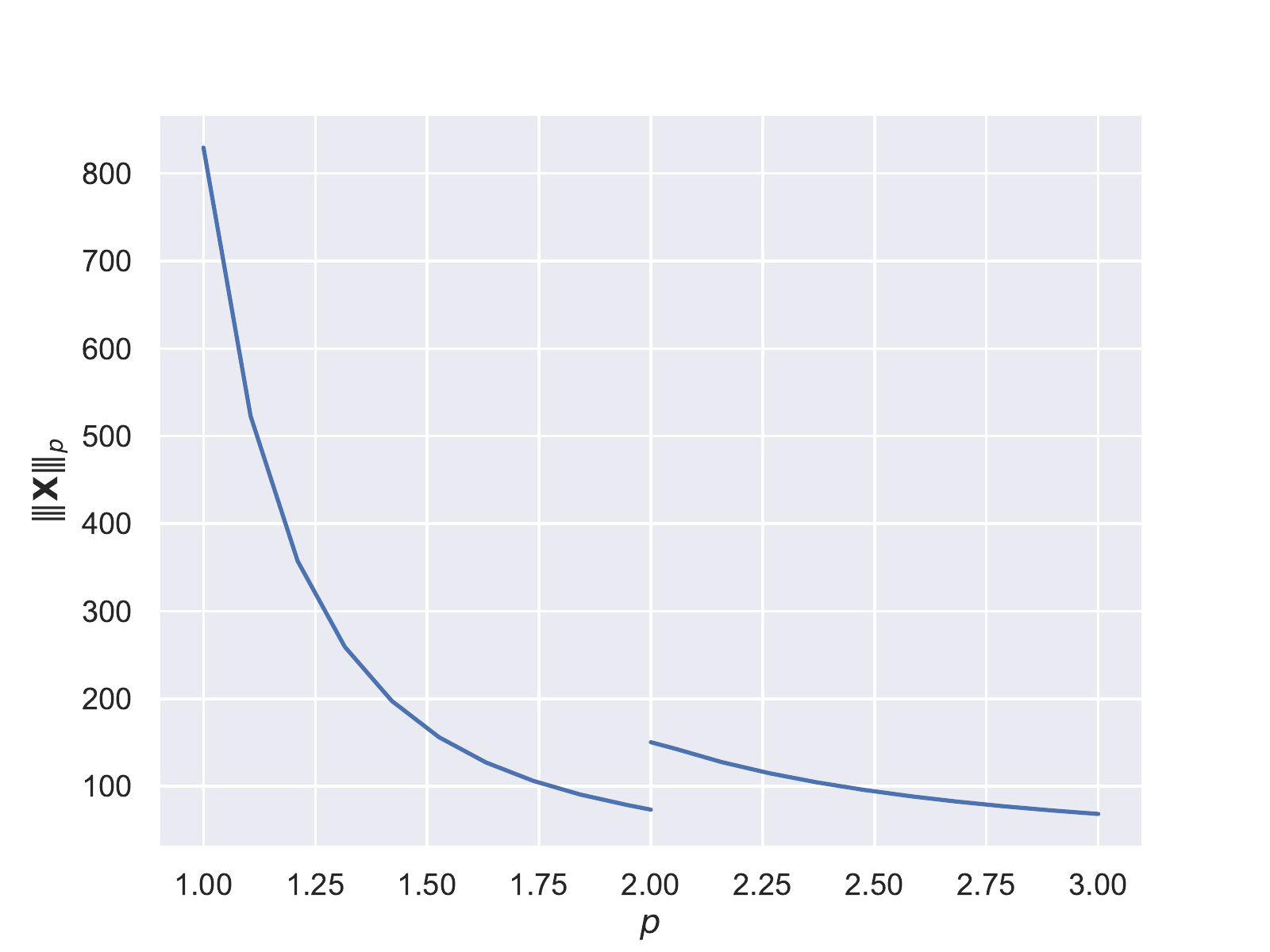}
    \caption{\(p\)-variation norm of the weights for \(p\in[1,3]\).}
    \label{fig:pvars}
  \end{subfigure}
  \caption{ResNet128 trained to MNIST data.}
  \label{fig:resnet-trained}
\end{figure}

The roughness of the driving weight matrices depicted in \Cref{fig:resnet-weights} might seem surprising at first sight. But recall the usual (random)
initialization practice of the weights before the start of training: weights are typically initialized to be independent across layers and nodes and,
in the case of constant dimension $d$, also identically distributed. There are several popular choices for the distribution itself, including normal
and uniform distributions. Hence, (possibly after a proper re-scaling reflecting a choice of ``time'', and possibly in some asymptotic sense) the path
of initialized weight matrices correspond to a matrix-valued Brownian motion, sampled in discrete time. As indicated by \Cref{fig:resnet-weights}, the
training does not seem to fundamentally change the picture: While trained weights are certainly no longer i.i.d., they still seem to exhibit the
roughness of sample paths of a Brownian motion. We refer to \cite{cohen2021scaling} for an in-depth study of scaling properties of deep residual
neural networks.

Accepting that the weights of deep residual neural networks behave like Brownian motions even after training, and considering the case of many layers
(e.g., $128$ layers as used in \Cref{fig:resnet-trained}), \Cref{fig:pvars} becomes clear. Indeed, paths of Brownian motion have finite $p$-variation
only for $p>2$ in the continuous time limit, hence we expect explosion of the $p$-variation for $p < 2$ even in the discrete case when the number of
steps becomes large. In particular, the $C^1$ analysis $p=1$ is expected to yield very poor results in the case of deep residual neural networks, if
no regularization techniques are used to enforce smoothness.

The article is organized as follows. In \Cref{s:1var} we review classical stability results from ODE theory and their application to the design of
stable residual architectures, and their counterparts in the discrete setting.
In \Cref{s:roughintro} we introduce the basic tools of discrete rough analysis needed in order to extend the previously mentioned results to the
\(p\)-variation topology.
Next, in \Cref{sse:lift} we review the algebraic theory of the so-called iterated-sums signature of a time series.
Finally in \Cref{s:cres} we prove stability bounds for residual architectures in the \(p\)-variation norms, for \(p\in[1,3)\).

\subsection*{Acknowledgments}

The authors gratefully acknowledge the support by the German research foundation DFG through the cluster of excellence MATH+, projects EF1-5 and EF1-13. We are also grateful for related discussions with Terry Lyons and Gitta Kutyniok.

\section{Stability in finite-variation norm}
\label{s:1var}
Classical analytical tools can be exploited to understand the behavior of deep ResNets by comparing their behavior to a limiting ODE system of the
form \cref{eqn:rnn.ode} \cite{W2017,HRH+2018,HR18}.
The main tool for this kind of analysis is Grönwall's inequality, which we now recall.
\begin{theorem}
	Let \(u,\alpha,\beta\colon[0,T]\to\R\) be continuous functions, with \(\alpha\) non-decreasing and \(\min(\alpha,0)\in L^1\), such that
	\[
		u(t)\le \alpha(t)+\int_0^t\beta(s)u(s)\,\mathrm ds
	\]
	for all \(t\in[0,T]\).
	Then
	\[
		u(t)\le\alpha(t)\exp\left( \int_0^t\beta(s)\,\mathrm ds \right)
	\]
	for all \(t\in[0,T]\).
\end{theorem}

It is a standard result that, together with a priori bounds for solutions to \cref{eqn:rnn.ode}, this result implies the following stability bound
\cite[Theorem 3.15]{FV10}.
\begin{theorem}
	Let \(\bx,\tilde\bx\) be solutions to the ordinary differential equation
	\[
		\frac{\mathrm d}{\mathrm dt}\bx(t)=f(\bx(t))\frac{\mathrm d}{\mathrm dt}\bw(t)
	\]
	started respectively from \(\xi,\tilde\xi\in\R^n\) and driven by \(\bw,\tilde\bw\in C^1([0,T],\R^d)\).
	If \(f\in\mathrm{Lip}(\R^n,\mathcal L(\R^d,\R^n))\), the bound
	\[
		\|\bx-\tilde\bx\|_{\infty;[0,T]}\le e^{2\|f\|_{\mathrm{Lip}}\|\tilde\bw\|_{1;[0,T]}}(|\xi-\tilde\xi|+\|f\|_{\mathrm{Lip}}\|\bw-\tilde\bw\|_{1;[0,T]})
	\]
\end{theorem}

The usefulness of these results in the previously mentioned references comes from the fact that for smooth enough driving signals, the behavior of
the discrete control system defined in \cref{eq:creintro} will be well approximated by the continuous-time limiting system.  However, this relies on
the assumption that the driving path \(\bw\) is indeed smooth, and that we are considering enough time steps, i.e., the network is deep enough.  It
turns out that in practice neither of these assumptions might be satisfied (see \Cref{fig:resnet-weights}).  The main goal of this paper is to show
that both these assumptions can be removed while retaining the stability results.

In this section we show how to obtain such a bound in the finite time-horizon regime, i.e., working directly at the discrete level. In the current
literature the smoothness assumption is sometimes circumvented by penalizing the \(L^1\) norm (or \(C^1\) in continuous-time models) of the weights
during training in order to enforce the necessary smoothness.  As before, the main tool is a discrete version of Grönwall's inequality (see e.g.
\cite[Lemma A.3]{K2014}).

\begin{theorem}
	\label{thm:discr.gronwall}
	Let \(c\ge 0\) and \(\varphi_j\) and \(v_j\) be non-negative sequences. If
	\[
		\varphi_j\le c+\sum_{i=1}^{j-1}v_i\varphi_i
	\]
	for all \(j\ge 1\), then
	\[
		\varphi_j\le c\prod_{i=1}^{j-1}(1+v_i)\le c\exp\left( \sum_{i=1}^{j-1}v_i \right)
	\]
	for all \(j\ge 1\).
\end{theorem}

Let us consider solutions \(\bx,\tilde\bx\) to \cref{eq:creintro}, driven resp. by
\(\bw,\tilde\bw\) and started resp. from two different initial conditions \(\xi,\tilde\xi\in\R^m\).  Suppose furthermore that the vector fields
\(f_\mu\) are Lipschitz and bounded. We denote by \(L(f)\) the Lipschitz constant of \(f\colon\R^m\to\R^m\).

Considering the difference \(\bz_k\coloneq\bx_k-\tilde\bx_k\) and letting \(\Delta_k=\bw_k-\tilde\bw_k\), we can immediately observe that
\begin{align*}
	\bz_{k+1}-\bz_k&= \bx_{k+1}-\bx_k-(\tilde\bx_{k+1}-\tilde\bx_k)\\
	&= \sum_{\mu=1}^df_\mu(\bx_k)(\bw_{k+1}^\mu-\bw_k^\mu)-\sum_{\mu=1}^df_\mu(\tilde\bx_k)(\tilde\bw^\mu_{k+1}-\tilde\bw^\mu_k).
\end{align*}
Therefore
\begin{align*}
	|\bz_{k+1}-\bz_k|\le\sum_{\mu=1}^d|f_\mu(\bx_k)|\left|\Delta^\mu_{k+1}-\Delta^\mu_k\right|+\sum_{\mu=1}^d|f_\mu(\bx_k)-f_\mu(\tilde\bx_k)||\tilde\bw^\mu_{k+1}-\tilde\bw^\mu_k|.
\end{align*}

Hence, we see that
\[
	|\bz_{k+1}-\bz_k|\le\|f\|_{\infty}|\Delta_{k+1}-\Delta_k|+L(f)|\bz_k||\tilde\bw_{k+1}-\tilde\bw_k|.
\]
Performing a telescopic sum we obtain that
\[
	|\bz_k|\le L(f)\sum_{j=0}^{k-1}|\bz_j||\tilde\bw_{j+1}-\tilde\bw_j|+|\bz_0|+\|f\|_{\infty}\sum_{j=0}^{k-1}|\Delta_{j+1}-\Delta_j|.
\]
The second term in the right-hand side is bounded by
\[
	|\bz_0|+\|f\|_{\infty}\|\Delta\|_{1;[0,N]}.
\]

Therefore, we obtain from \Cref{thm:discr.gronwall} that
\begin{equation*}
	|\bz_k|\le (|\bx_0-\tilde\bx_0|+\|f\|_{\infty}\|\Delta\|_{1;[0,N]})\prod_{j=0}^{k-1}(1+L(f)|\tilde\bw_{j+1}-\tilde\bw_j|).
\end{equation*}
Using the elementary estimate \(1+x\le e^x\) we may finally obtain
\begin{equation}
\label{eq:1.gronw}
	\sup_{k=0,\dotsc,N}|\bx_k-\tilde\bx_k|\le e^{L(f)\|\tilde\bw\|_{1;[0,N]}}(|\bx_0-\tilde\bx_0|+\|f\|_\infty\|\bw-\tilde\bw\|_{1;[0,N]}).
\end{equation}

Therefore, it is possible to obtain Lipschitz stability results at the discrete-time level.
In the continuous-time case, it is possible to prove such theorems with respect to the whole range of \(p\)-variation topologies, for any
\(p\in[1,\infty)\).
In the rest of the article we introduce the analogous techniques for treating the discrete-time case and we show how to obtain the desired bounds for
\(p\in[1,3)\).
The main difficulty in this case is that \Cref{thm:discr.gronwall} is not well adapted to the weaker topologies, so a new generalization is needed (see \Cref{prop:gronwall}).
Indeed, directly applying	\Cref{thm:discr.gronwall} in the \(p\)-variation norm would lead to a bound like \cref{eq:1.gronw} constant depending on \(N\), which is unbounded as \(N\to\infty\).

\section{Elements of rough analysis}
\label{s:roughintro}
We begin with a brief overview of classical results present in the rough analysis literature.
We remark that many of these results are usually stated in terms of continuous-time variables which
introduces certain additional difficulties.
In our case, no such difficulties arise so the statements and proofs of analogous results become
simpler.

\subsection{Discrete controls}
We recall that in the setting of \cite{Lyo98} a \emph{control function} (or simply a
\emph{control}) is a function \(\omega\colon [0,\infty)\times[0,\infty)\to[0,\infty)\) which is
super-additive, in the sense that \(\omega(s,u)+\omega(u,t)\le \omega(s,t)\) for all \(s<u<t\).
In the continuous-time setting, the main motivation for introducing control functions is
to measure the size of the increments of a function in a more flexible way than what the
natural control \(\omega(s,t)=|t-s|\) allows.

\begin{definition}[\cite{Dav08}]
  A (discrete) control is a triangular array of non-negative numbers \((\omega_{k,l}:k<l)\) such that
  \(\omega_{k,k}=0\) and
  \[ \omega_{k,l}+\omega_{l,m}\le \omega_{k,m} \]
  for all \(k<l<m\)
\end{definition}

\begin{remark}
  Observe that for a control $\omega$ the maps \(l\mapsto \omega_{k,l}\) and \(k\mapsto \omega_{k,l}\) are
  non-decreasing and non-increasing, respectively.
  Indeed, if \(0\le k<l<m\le N\) then
  \[ \omega_{k,l}\le \omega_{k,l}+\omega_{l,m}\le \omega_{k,m} \]
  and
  \[ \omega_{k,m}\ge \omega_{k,l}+\omega_{l,m}\ge \omega_{l,m}. \]
  \label{rem:control.monotone}
\end{remark}

Now we collect some results on how to produce new controls out of any given control.

\begin{lemma}
  Let \(\omega\) be a control and \(\varphi\colon[0,\infty)\to[0,\infty)\) an increasing convex
  function such that \(\varphi(0)=0\). Then \(\tilde{\omega}_{k,l}\coloneq\varphi(\omega_{k,l})\) is
  also a control.
  \label{lem:control.convex}
\end{lemma}
\begin{proof}
  Since \(\varphi\) is convex and \(\varphi(0)=0\) we have that
  \[ \varphi(\lambda(x+y))\le \lambda\varphi(x+y) \]
  for any \(\lambda\in[0,1]\).
  Choosing \(\lambda=\frac{x}{x+y}\) we obtain
  \[ \varphi(x)\le \frac{x}{x+y}\varphi(x+y). \]
  Similarly, \(\varphi(y)\le\tfrac{y}{x+y}\varphi(x+y)\) so that
  \[ \varphi(x)+\varphi(y)\le \varphi(x+y), \]
  i.e. \(\varphi\) is super-additive.

  Therefore, if \(0\le k<l<m\le N\),
  \begin{align*}
    \tilde{\omega}_{k,l}+\tilde{\omega}_{l,m}&= \varphi(\omega_{k,l})+\varphi(\omega_{l,m})\\
    &\le \varphi(\omega_{k,l}+\omega_{l,m})\\
    &\le \varphi(\omega_{l,m})=\tilde{\omega}_{l,m}
  \end{align*}
  where the last inequality follows from the monotonicity of \(\varphi\).
\end{proof}

\begin{remark}
  In particular, this implies that if \(\omega\) is a control, then \(\omega^\alpha\) is
  also a control, for any \(\alpha>1\).
\end{remark}

\begin{lemma}
  Let \(\omega,\tilde{\omega}\) be two controls. If \(\alpha,\beta>0\) are such that \(\alpha+\beta\ge
  1\), then \(\hat{\omega}_{k,l}\coloneq \omega_{k,l}^\alpha\tilde{\omega}_{k,l}^\beta\) is also a control.
  \label{lem:control.product}
\end{lemma}
\begin{proof}
  Let \(\theta\coloneq\alpha+\beta\).
  By \Cref{lem:control.convex}, it is enough to show that
  \[ z_{k,l}\coloneq \omega_{k,l}^{\tfrac\alpha\theta}\tilde{\omega}_{k,l}^{\tfrac\beta\theta} \]
  is a control, since then \(\hat{\omega}_{k,l}=z_{k,l}^\theta\) will also be a control.
  Since \(\tfrac\alpha\theta+\tfrac\beta\theta=1\), Hölder's inequality implies that
  \begin{align*}
    z_{k,l}+z_{l,m}&\le(\omega_{k,l}+\omega_{l,m})^{\tfrac\alpha\theta}(\tilde{\omega}_{k,l}
      +\tilde{\omega}_{l,m})^{\tfrac\beta\theta}\\
    &\le \omega_{k,m}^{\tfrac\alpha\theta}\tilde{\omega}_{k,m}^{\tfrac\beta\theta}
  \end{align*}
  and the proof is finished.
\end{proof}

\subsection{\texorpdfstring{\(p\)-variation}{p-variation}}
In the following we will deal with \emph{time series}, which are finite sequences of
vectors \(\bw=(\bw_0,\bw_1,\dotsc,\bw_N)\in(\R^d)^N\).
We will use the convention of indexing time steps with lower indices and lowercase Latin letters, and spatial components with
upper indices and lowercase Greek letters, so for example \(\bw^\mu_k\in\R\) refers to the \(\mu\)-th component of the
\(k\)-th entry in the time series \(\bw\).
The main reason for making this distinction is that the ranges for both sets of variables is different: indeed, note that Greek letter indices always
belong to the set \(\{1,\dotsc,d\}\), while Latin letter indices belong to the set \(\{0,\dotsc,N\}\).

We will also need to deal with general \emph{triangular arrays}, which are collections of
vectors of the form \((\Xi_{k,l}:0\le k<l\le N)\).
For any time series we define a triangular array \((\bw_{k,l})\) by setting
\(\bw_{k,l}\coloneqq \bw_l-\bw_k\).

\begin{definition}
    Given \(p>0\), we define the \(p\)-variation with respect to a fixed choice of norm
    \(\lvert\,\cdot\,\rvert\) on \(\R^d\), by
    \[ \|\bw\|_{p;[k,l]}\coloneqq\left( \max_{s\in\mathcal S_{k,l}} \sum_{j=0}^{\#s}|\bw_{s_{j+1}}-\bw_{s_j}|^p \right)^{1/p} \]
    where the maximum is taken over the set \(\mathcal S_{k,l}\) of all increasing subsequences
    \[ s=(s_0=k,s_1,\dotsc,s_m,s_{m+1}=l) \]
    of \(\{k,k+1,\dotsc,l-1,l\}\) and we have set \(\#s=m\) for such a sequence.
For a triangular array \(\Xi\) one can also define its \(p\)-variation as
\[ \|\Xi\|_{p;[k,l]}\coloneqq\left( \sup_{s\in\mathcal S_{k,l}}\sum_{j=0}^{\#s}|\Xi_{s_j,s_{j+1}}|^p \right)^{1/p}. \]
\end{definition}
We observe that in the case where \(\Xi_{k,l}=\bw_l-\bw_k\) both definitions coincide.

Since the trivial sequence \((k,l)\in\mathcal S_{k,l}\) we obtain immediately the bound
\begin{equation}
  \lvert\Xi_{k,l}\rvert\le \lVert\Xi\rVert_{p;[k,l]}\label{eq:pvarbound}
\end{equation}
for any \(p>0\).
In the particular case where \(\Xi_{k,l}=\bw_l-\bw_k\) we also obtain
\[ \|\bw\|_\infty\coloneq\sup_{k=0,\dotsc,N}|\bw_k|\le|\bw_0|+\|\bw\|_{p;[0,N]}. \]

\begin{proposition}
  Let \(\Xi\) be a triangular array and \(p\ge 0\).
  Then \(\omega_{k,l}\coloneqq\|\Xi\|_{p;[k,l]}^p\) is a control.
\end{proposition}
\begin{proof}
  Indeed, if \(s'\in\mathcal S_{k,l}\) and \(s''\in\mathcal S_{l,m}\) then \(s=(s',s'')\in\mathcal S_{k,m}\) and so
  \[ \sum_{j=0}^{\#s'}|\Xi_{s_j,s_{j+1}}|^p+\sum_{j'=0}^{\#s''}|\Xi_{s'_{j'},s'_{j'+1}}|^p \le \|\Xi\|^p_{p;k,m} \]
  and super-additivity follows from taking the supremum over \(\mathcal S_{k,l}\) and
  \(\mathcal S_{l,m}\).
\end{proof}

\begin{remark}
  Since the set \(\mathcal S_{k,l}\) is finite, the \(p\)-variation norm of \(\Xi\) is
  finite for any \(p>0\) and triangular array \(\Xi\).
  This should be contrasted with the usual setting for rough paths, where one deals with paths in
  continuous time; in that setting, the \(p\)-variation norm can become infinite and this introduces
  a number of analytical problems which are not present in the present context.
\end{remark}
\begin{remark}
    The \(p\)-variation defines a quasi-norm for \(0<p<1\) (i.e. the triangle inequality fails), and
    a semi-norm for \(p\ge 1\) on time series, since all constant sequences have vanishing
    \(p\)-variation. For \(p\ge 1\), it becomes a norm on triangular arrays.
\end{remark}

\begin{lemma}
  Let \(0\le p<q<\infty\).
  Then \( \lVert\Xi\rVert_{q;[k,l]}\le\lVert\Xi\rVert_{p;[k,l]} \)
  \label{lem:pvar.monotone}
\end{lemma}
\begin{proof}
  Observe that, since \(\tfrac qp>1\), the inequality
  \[
    \sum_{j=0}^{\#s}|\Xi_{s_j,s_{j+1}}|^q\le\left( \sum_{j=0}^{\#s}|\Xi_{s_j,s_{j+1}}|^{p}
    \right)^{q/p}
  \]
  holds for any \(s\in\mathcal S_{k,l}\).
\end{proof}

Given a triangular array \(\Xi\), we define another collection \((\delta \Xi_{k,l,m}:0\le k<l<m)\) by
\[ \delta \Xi_{k,l,m}\coloneqq \Xi_{k,m}-\Xi_{k,l}-\Xi_{l,m}. \]
In the special case where \(\Xi_{k,l}=\bw_l-\bw_k\) we see that \(\delta\Xi_{k,l,m}=0\).
The operator \(\delta\) satisfies the following product rule: if \(\bw\) is a time series
and \(\Xi\) is a triangular array, consider the triangular array \(\mathbf
Z_{k,l}\coloneq\bw_k\Xi_{k,l}\). Then
\begin{equation}
  \delta\mathbf Z_{k,l,m}=\bw_k\delta\Xi_{k,l,m}-\bw_{k,l}\Xi_{l,m}.
  \label{eq:deltader}
\end{equation}

Finally we collect here some standard results for further reference.

\begin{lemma}
  Let \(\Xi\) be a triangular array and \(p\ge 0\). Suppose there is a control \(w\) such that
  \[ |\Xi_{k,l}|\le C\omega_{k,l}^{1/p} \]
  for all \(0\le k<l\le N\) and some constant \(C>0\).
  Then,
  \[ \|\Xi\|_{p;[k,l]}\le C\omega_{k,l}^{1/p} \]
  for all \(0\le k<l\le N\).
  \label{lem:pvar.control}
\end{lemma}
\begin{proof}
  By hypothesis the inequality
  \[ |\Xi_{k,l}|^p\le C^p\omega_{k,l} \]
  holds for all \(0\le k<l\le N\).
  By superadditivity of \(w\), if \(s\in\mathcal S_{k,l}\) then also
  \[ \sum_{j=0}^{\#s}|\Xi_{s_j,s_{j+1}}|^p\le C^p\omega_{k,l}. \]
  The desired bound follows upon taking the maximum over \(s\in\mathcal S_{k,l}\).
\end{proof}
\begin{lemma}
  Assume that \(p\ge 1\) and
  \[ |\bx_{k,l}|\le C\omega_{k,l}^{1/p} \]
  for all \(0\le k<l\) such that \(\omega_{k,l}\le 1\) or if \(l=k+1\).
  Then
  \[ \|\bx\|_{p;[k,l]}\le 2C(\omega_{k,l}^{1/p}\vee \omega_{k,l}) \]
  for all \(0\le k<l\).
    \label{lem:maxbound}
\end{lemma}
\begin{proof}
    We show that the inequality \(|\bx_{k,l}|\le C\tilde\omega_{k,l}^{1/p}\) holds for all \(0\le k<l\le
    N\), where \(\tilde\omega_{k,l}\coloneq\omega_{k,l}^p\vee\omega_{k,l} \) which is a control by
    \Cref{lem:control.convex}.
    The conclusion then follows from \Cref{lem:pvar.control}.

    If \(k<l\) are such that \(\omega_{k,l}\le 1\) then there is nothing to show, since in this case
    \(\tilde\omega_{k,l}^{1/p}=\omega_{k,l}^{1/p}\).
    Suppose now that \(k<l\) are such that \(\omega_{k,l}>1\).
    Inductively define \(j_0=k<j_1<\dotsb<j_M<j_{M+1}=l\) by setting
    \[ j_{u+1}\coloneq\max\{j>j_u:\omega_{j_u,j}\le 1\}\wedge(j_u+1). \]
    By super-additivity of \(\omega\) we immediately get that \(M+1\le 2\omega_{k,l}\).
    Also, by definition \(|\bx_{j_u,j_{u+1}}|\le C\omega_{j_u,j_{u+1}}^{1/p}\) for
    \(u=0,1,\dotsc,r\).
    Thus, by the triangle inequality we obtain that
    \begin{align*}
			|\bx_{k,l}|&\le C\sum_{u=0}^M\omega_{j_u,j_{u+1}}^{1/p}\\
				&\le C(M+1)\\
				&\le 2C\omega_{k,l}\\
				&= 2C\tilde\omega_{k,l}^{1/p}.\qedhere
    \end{align*}
\end{proof}

Finally, we show the following result, known as the rough Grownall Lemma.
It is a slight variation of \cite[Lemma 2.12]{DGH+2019}, adapted to our particular setting.
\begin{theorem}
	\label{prop:gronwall}
    Let \(\bz\) be a time series and suppose there exist controls \(\omega,\tilde\omega\) such that
    \[ |\bz_{k,l}|\le C\left( \max_{j=0,\dotsc,l}|\bz_j| \right)\omega_{k,l}^{1/\kappa}+\tilde\omega_{k,l}^{1/\rho}
    \]
    whenever \(\omega_{k,l}\le L\) or \(l=k+1\), for some constants \(C>0\) and \(\kappa,\rho\ge1\).
    Then,
    \[
        \max_{j=0,\dotsc,N}|\bz_j|\le 2\exp\left( \frac{\omega_{0,N}}{\alpha L} \right)\left\{
				|\bz_0|+\max_{j=0,\dotsc,N}\left(\tilde{\omega}^{1/\rho}_{0,j}\left( 1+2\frac{\omega_{0,j}}{\alpha L} \right)^{1-1/\rho}\exp\left(- \frac{\omega_{0,j}}{\alpha L} \right) \right) \right\}
    \]
    where \(\alpha\coloneq\min(1,\frac1{L(2Ce^2)^\kappa})\).
\end{theorem}
\begin{proof}
    Define the sequences
    \[ G_k\coloneq\max_{j=0,\dotsc,k}|\bz_j|,\enspace H_k\coloneq G_k\exp\left(
        -\frac{\omega_{0,k}}{\alpha L} \right),\enspace H^*_k\coloneq\max_{j=0,\dotsc,k}H_j.
    \]
		Subdivide the interval \(\{0,\dotsc,N\}\) into \(j_0=0<j_1<\dotsb<j_K<j_{K+1}=N\) where \(j_u\) is the largest integer in \(\{j_{u-1}+1,\dotsc,N\}\) such that \(\omega_{j_{u-1},j_u}\le\alpha L\) or \(j_u=j_{u-1}+1\) if such an integer does not exist.
		We note that by subadditivity we necessarily have, for each \(u=1,\dotsc,K\), that \[u\le1+2\frac{\omega_{0,j_u}}{\alpha L}.\]
		Indeed, by definition we have that for each \(r\), \(\omega_{j_{r-1},j_r+1}>\alpha L\), hence if \(j\in\{j_{u-1}+1,\dotsc, j_u\}\) we have
		\[
			0\le\omega_{j_{u-1}+1,j}\le2\omega_{0,j}-\sum_{r=0}^{u-2}\omega_{j_r,j_{r+1}+1}\le 2\omega_{0,j}-\alpha L(u-1),
		\]
		that is,
		\[
			u\le1+2\frac{\omega_{0,j}}{\alpha L}.
		\]

    Now, for \(j_{u-1}<j\le j_u\) we have
    \begin{align*}
        |\bz_{0,j}|&\le \sum_{r=0}^{u-2}|\bz_{j_r,j_{r+1}}|+|\bz_{j_{u-1},j}|\\
        &\le
				\sum_{r=0}^{u-2}\left( CG_{t_{r+1}}\omega_{j_r,j_{r+1}}^{1/\kappa}+\tilde\omega_{j_r,j_{r+1}}^{1/\rho} \right)+CG_j\omega_{j_{u-1},j}^{1/\kappa}+\tilde{\omega}_{j_{u-1},j}^{1/\rho}\\
				&\le C(\alpha L)^{1/\kappa}\sum_{r=0}^{u-1}G_{j_{r+1}}+u^{1-1/\rho}\tilde{\omega}_{0,j}^{1/\rho}.
    \end{align*}

    We bound the first term on the right-hand side by
    \begin{align*}
        \sum_{r=0}^{u-1}G_{j_{r+1}}&= \sum_{r=0}^{u-1}H_{j_{r+1}}\exp\left(\frac{\omega_{0,j_{r+1}}}{\alpha
        L} \right)\\
        &\le H^*_N\sum_{r=1}^u\mathrm e^r\\
        &\le H^*_N\mathrm e^{u+1}.
    \end{align*}
    
    Combining this with the previous bound we obtain
		\[ G_j\le|\bz_0|+C(\alpha L)^{1/\kappa}\mathrm e^{u+1}H^*_N+u^{1-1/\rho}\tilde{\omega}^{1/\rho}_{0,j} \]
    and so
		\[ H_j\le\left( |\bz_0|+\tilde{\omega}_{0,j}^{1/\rho}\left(1+2\frac{\omega_{0,j}}{\alpha L}\right)^{1-1/\rho} \right)\exp\left( -\frac{\omega_{0,j}}{\alpha L}
    \right)+C(\alpha L)^{1/\kappa}\mathrm e^2H_N^*. \]
    This implies that
		\[ H_N^*\le|\bz_0|+\max_{j=0,\dotsc,N}\left\{ \tilde{\omega}_{0,j}^{1/\rho}\left( 1+2\frac{\omega_{0,j}}{\alpha L} \right)^{1-1/\rho}\exp\left(
        -\frac{\omega_{0,j}}{\alpha L} \right) \right\}+C(\alpha L)^{1/\kappa}\mathrm e^2H_N^*
    \]
    and so, by our choice of \(\alpha\) we obtain
    \begin{align*}
        \max_{j=0,\dotsc,N}|\bz_j|=G_N&\le H_N^*\exp\left( \frac{\omega_{0,N}}{\alpha L} \right)\\
				&\le 2\exp\left( \frac{\omega_{0,N}}{\alpha L} \right)\left\{ |\bz_0|+ \max_{j=0,\dotsc,N}\left( \tilde{\omega}_{0,j}^{1/\rho}\left( 1+2\frac{\omega_{0,j}}{\alpha L} \right)^{1-1/\rho}\exp\left(
        -\frac{\omega_{0,j}}{\alpha L} \right) \right)\right\}
    \end{align*}
    and we are done.
\end{proof}

\subsection{The Sewing Lemma}
At the core of the theory of rough integration lies the Sewing Lemma \cite{FL2006,Gub04}.
Therefore, it is tightly connected with the solution theory of differential equations driven by rough
signals.
Since our main aim is to perform a precise analysis of the behavior of discrete equations driven by
irregular time-series, it is no doubt that its discrete analogue will play a prominent rôle here as
well.

We begin by showing some preliminary results.
\begin{lemma}
  Suppose \(s\in\mathcal S_{k,l}\) of length \(\#s=m\).
For any given control \(\omega\), there exists an integer \(j^*\) with \(1\le j^*\le m\) such that
\[ \omega_{s_{j^*-1},s_{j^*+1}}\le\frac2m\omega_{k,l}. \]
  \label{lemma:coarse}
\end{lemma}
\begin{proof}
  Suppose, on the contrary, that for any \(1\le j\le m\) we have that
  \[ \omega_{s_{j-1},s_{j+1}}>\frac2m\omega_{k,l}. \]
  Then this would imply that
  \[ 2\omega_{k,l}<\sum_{j=1}^m\omega_{s_{j-1},s_{j+1}}\le 2\omega_{k,l} \]
  by super-additivity, which is a contradiction.
\end{proof}

\begin{proposition}[Discrete sewing]
  Let \((\Xi_{k,l}:0\le k\le l\le N)\) be a triangular array, and suppose that there exist two controls \(\omega\)
  and \(\tilde \omega\) such that
  \[ \lvert\delta\Xi_{k,l,m}\rvert\le \omega_{k,l}^\alpha\tilde{\omega}_{l,m}^\beta \]
	for some \(\alpha,\beta>0\) with \(\alpha+\beta>1\), and for all \(0\le k<l\le N\).
  Then
  \[ \left\lvert \sum_{j=k}^{l-1}\Xi_{j,j+1}-\Xi_{k,l} \right\rvert\leq
  2^{(\alpha+\beta)}\zeta_N(\alpha+\beta)\omega_{k,l}^\alpha\tilde{\omega}_{k,l}^\beta \]
	for all \(0\le k<l\le N\), where \(\zeta_N\) denotes the partial sum of Riemann's zeta function
	\[
		\zeta_N(s)\coloneq\sum_{n=1}^Nn^{-s}.
	\]
  \label{proposition:sew}
\end{proposition}
\begin{proof}
  By \Cref{rem:control.monotone} we deduce that \(|\delta\Xi_{k,l,m}|\le
  \omega_{k,m}^{\alpha}\tilde{\omega}_{k,m}^{\beta}\), and \Cref{lem:control.product} implies that
  \(\hat{\omega}\coloneqq \omega^{\tfrac\alpha\theta}\tilde{\omega}^{\tfrac\beta\theta}\) is a control.

  Now we apply a Young-style argument to estimate the above difference.
  First we observe that if \(l-k=1\) then the bound is trivial since the left-hand side vanishes.
  Therefore we assume that \(l-k\ge2\).
  By \Cref{lemma:coarse} we can find an index \(k<j^*<l\) such that \[
  \hat{\omega}_{j^*-1,j^*+1}\leq\frac2{(l-k-1)}\hat{\omega}_{k,l}. \]
  Hence, if we denote by \(s\coloneqq(k,k+1,\dotsc,j^*-1,j^*+1,\dotsc,l)\) we have
  \[ \left\lvert\sum_{j=k}^{l-1}\Xi_{j,j+1}-\sum_{s}\Xi_{s_j,s_{j+1}} \right\vert =
  |\delta\Xi_{j^*-1,j^*,j^*+1}|\le\left( \frac2{l-k-1} \right)^\theta\hat{\omega}_{k,l}^\theta. \]
  Then we can apply \Cref{lemma:coarse} again to the sequence \(s\) to obtain a
  ``coarser'' sequence \(s'\), containing one less point, and such that
  \[ \left\lvert\sum_s\Xi_{s_j,s_{j+1}}-\sum_{s'}\Xi_{s'_j,s'_{j+1}}\right\rvert\le\left(
  \frac2{l-k-2} \right)^\theta \hat{\omega}^\theta_{k,l}. \]
  Continuing in this way we obtain a sequence of coarsenings of the full sequence until we
  get to \(s^*=(k,l)\), and by using the triangular inequality we then deduce the estimate
  \[ \left\lvert \sum_{j=k}^{l-1}\Xi_{j,j+1}-\Xi_{k,l}
  \right\rvert\le2^\theta\sum_{r=1}^{l-k-1}\frac1{r^\theta}\hat{\omega}_{k,l}^\theta \]
  from where the conclusion follows.
\end{proof}

We will also need the following generalization of the Sewing Lemma, whose proof is straightforward.
\begin{proposition}[Generalized discrete sewing]
  Suppose that \(\Xi\) is a triangular array as before. Suppose that there are controls
  \(\omega_{r}\) and \(\tilde \omega_{r}\), and exponents \(\alpha_r,\beta_r>0\) such that
  \(\alpha_r+\beta_r>1\) for all \(r=1,\dotsc,n\).
  If
  \[ \lvert\delta\Xi_{k,l,m}\rvert\le\sum_{r=1}^n\omega_{r;k,l}^{\alpha_r}\tilde \omega_{r;l,m}^{\beta_r} \]
  then
  \[ \left\lvert \sum_{j=k}^{l-1}\Xi_{j,j+1}-\Xi_{k,l} \right\rvert\le
  2^{\hat\theta}\zeta_N(\hat{\theta})\sum_{r=1}^n\omega_{r;k,l}^{\alpha_r}\tilde \omega_{r;k,l}^{\beta_r} \]
  where \(\hat{\theta}\coloneq\displaystyle\min_{r=1,\dotsc,n}\{\alpha_r+\beta_r\}\).
  \label{proposition:gensew}
\end{proposition}

\section{Lifting time series}
\label{sse:lift}
Inspired by the theory of rough paths, we introduce an augmentation or lift of a given time series \(\bw\).
Recall that the convention of using lowercase Latin letters as sub-indices to index time, and lowercase Greek letters to index spatial components is
in place.
\begin{definition}
	Given a time series \(\bw\), we call its \emph{lift} the triangular array of \(d\)-by-\(d\) matrices \(\bbW\) defined by
	\[
		\bbW_{k,l}^{\mu\nu}\coloneq\sum_{j=k}^{l-1}(\bw^\mu_j-\bw^\mu_k)(\bw^\nu_{j+1}-\bw^\nu_{j+1}).
	\]
	We write \(\bW\coloneq(\bw,\bbW)\).
\end{definition}

The main purpose of this lift is to provide ``second order information'' about the time series.
It is, first of all, a discrete analogue of an iterated integral as in the rough path setting, but it may be interpreted as a generalized quadratic
covariation of the components of \(\bw\).
The lift \(\bbW\) is part of a much larger structure, known as the \emph{iterated-sums signature} of \(\bw\) \cite{DET2020}.

We now record a basic property of \(\bbW\) for later use:
\begin{theorem}
	\label{thm:chen}
	The time series lift \(\bbW\) of a time series \(\bw\) satisfies \emph{Chen's identity}: for all indices \(0\le k<l<m\le N\) and \(\mu,\nu\in\{1,\dotsc,d\}\) we have
			\[
				\delta\bbW^{\mu\nu}_{k,l,m}=\bw^\mu_{k,l}\bw^\nu_{l,m}.
			\]
\end{theorem}

Given \(p\in[2,3)\), a pair \(\bW=(\bw,\bbW)\) consisting of a time series and its lift, and indices \(0\le k\le l\le N\), we define a semi-norm
\begin{equation}
\label{eq:hom.norm}
	\t|\bW|_{p;[k,l]}\coloneq\|\bw\|_{p;[k,l]}+\|\bbW\|^{1/2}_{p/2;[k,l]},
\end{equation}
and a pseudometric
\begin{equation}
\label{eq:pseudometric}
	\rho_p(\bW,\tilde\bW)\coloneq \|\bw-\tilde\bw\|_{p;[0,N]}+\|\bbW-\tilde\bbW\|_{p/2;[0,N]}.
\end{equation}
We note that both can be turned into a proper norm (resp. metric) if we add the absolute value of the initial value.

\section{Controlled difference equations}
\label{s:cres}
In this section we consider equations of the form
\begin{equation}
  \bx_{k+1}=\bx_k+\sum_{\mu=1}^df_\mu(\bx_k)(\bw^\mu_{k+1}-\bw^\mu_k),\quad\bx_0=\xi\in\R^m
  \label{eq:cre}
\end{equation}
for some vector fields \(f_1,\dotsc,f_d\) on \(\R^m\), and where \(k\) ranges between \(0\) and
some fixed time horizon \(n\in\N\).
Our main aim is to obtain some control over the size of the end-point value \(\bx_n\) of the
solution.

In view of the previous sections, and in particular of the bound in \cref{eq:pvarbound},
we will try to obtain good estimates for the \(p\)-variation norm \(\|\bx\|_{p;[0,n]}\).
Of course, such estimates will require some assumptions on the vector fields.
It turns out that we will not only be able to control the ``large scale'' behavior of
\(\bx\), but we will also obtain a cascade of estimates of some remainder terms, reminiscent
of a Taylor expansion.

The techniques needed to obtain those bounds will depend crucially on \(p\in[1,\infty)\).
At first, we distinguish two basic regimes: \(p\in[1,2)\) and \(p\in[2,\infty)\).
By analogy with the rough paths literature, we call the former the \emph{young regime}, and
the latter the \emph{rough regime}~--~even though there is strictly no notion of roughness
in our setting.
The rough regime can be further subdivided into the cases where \(p\in[n,n+1)\), which we
call the \emph{level n rough regime}.
The terminology will make itself clear later down the road.

A central tool for constructing solutions to ODEs driven by rough paths are the so-called
\emph{controlled paths}, introduced by Gubinelli \cite{Gub04}.
See also \cite{HK2015}.
In a nutshell, the notion of ``controlledness'' contains all the necessary analytical
estimates needed for the definition of a rough integral which then is used to give sense
to solutions of Rough Differential Equations.
In the present setting no such definition is needed since there are no divergences
appearing from considering \cref{eq:cre}.
Nonetheless, we can still derive similar bounds.
Note however that in our case the estimates are \emph{proven} rather than \emph{assumed}.

Given a vector field \(f\colon\R^m\to\R^m\) of class \(\cC^n_\b\), i.e. it and all its derivatives
up to order \(n\) are bounded, we define
\[ \|f\|_{\cC^n_\b}\coloneq\max_{k=1,\dotsc,n}\|D^kf\|_{\infty}. \]
If \(f=(f_1,\dotsc,f_d)\) is a collection of vector fields on \(\R^n\) of class \(\cC^n_\b\) (or,
equivalently, a map in \(\cC^n_\b(\R^n,\R^{dn})\)), we define
\[ \|f\|_{\cC^n_\b}\coloneq\max_{\mu=1,\dotsc,d}\|f_\mu\|_{\cC^n_\b}. \]

\begin{lemma}\label{lem:f.bound}
	Suppose \(f\in\cC^2_\b\) and let \(\bx,\tilde\bx\) be two time series. Then
	\[
		\|f(\bx)-f(\tilde\bx)\|_{p;[k,l]}\le 2^{(p-1)/p}\|f\|_{\cC^2_\b}\left(
		\|\bx-\tilde\bx\|^p_{p;[k,l]}+\|\tilde\bx\|^p_{p;[k,l]}\|\bx-\tilde\bx\|^p_{\infty;[0,l]} \right)^{1/p}.
	\]
	Furthermore, if \(f\in\cC^3_\b\) and we let
	\[
		T_{k,l}\coloneq f(\bx_l)-f(\bx_k)-Df(\bx_k)\delta\bx_{k,l}
	\]
	and similarly for \(\tilde\bx\), then
	\[
		\|T-\tilde T\|_{p/2;[k,l]}\le2^{(p-2)/p}\|f\|_{\cC^3_\b}\left[
		\|\bx-\tilde\bx\|^p_{p;[k,l]}(\|\bx\|^p_{p;[k,l]}+\|\tilde\bx\|^p_{p;[k,l]})^{1/2}+\|\tilde\bx\|_{p;[k,l]}^p\|\bx-\tilde\bx\|^{p/2}_{\infty;[0,l]} \right]^{2/p}
	\]
\end{lemma}
\begin{proof}
	Suppose first that \(f\in\cC^1_\b\). By the Fundamental Theorem of Calculus we may write
	\[
		f(\bx_l)-f(\bx_k)=\int_0^1Df(\bx_k+\tau\delta\bx_{k,l})\delta\bx_{k,l}\,\mathrm d\tau.
	\]
	Therefore, by adding and subtracting cross terms, we see that
	\begin{align*}
		|f(\bx_l)-f(\bx_k)-(f(\tilde\bx_l)-f(\tilde\bx_k))|&\le \begin{multlined}[t]\int_0^1|Df(\bx_k+\tau\delta\bx_{k,l})(\delta\bx_{k,l}-\delta\tilde\bx_{k,l})|\,\mathrm d\tau
		\\+\int_0^1|(Df(\bx_k+\tau\delta\bx_{k,l})-Df(\tilde\bx_k+\tau\delta\tilde\bx_{k,l}))\delta\tilde\bx_{k,l}|\,\mathrm d\tau\end{multlined}
	\end{align*}
	The right-hand side is bounded by
	\[
		\|f\|_{\cC^2_\b}(\|\bx-\tilde\bx\|_{p;[k,l]}+\|\tilde\bx\|_{p;[k,l]}\|\bx-\tilde\bx\|_{\infty;[0,l]})\le
		2^{1-1/p}\|f\|_{\cC^2_\b}\left( \|\bx-\tilde\bx\|_{p;[k,l]}^p+\|\tilde\bx\|_{p;[k,l]}^p\|\bx-\tilde\bx\|_{\infty;[0,l]}^p \right)^{1/p},
	\]
	and the result follows from \Cref{lem:pvar.control}.

	Now, assume that \(f\in\cC^2_\b\). By iterated application of the Fundamental Theorem of Calculus we may now write
	\[
		T_{k,l}=\int_0^1\int_0^\tau D^2f(\bx_k+\nu\delta\bx_{k,l})(\delta\bx_{k,l},\delta\bx_{k,l})\,\mathrm d\nu\mathrm d\tau.
	\]
	Inserting appropriate cross terms we obtain
	\begin{align*}
		|T_{k,l}-\tilde T_{k,l}|\le\begin{multlined}[t]\int_0^1\int_0^\tau \left\lvert
			D^2f(\bx_{k}+\nu\delta\bx_{k,l})(\delta\bx_{k,l},\delta\bx_{k,l})-D^2f(\bx_{k}+\nu\delta\bx_{k,l})(\delta\tilde\bx_{k,l},\delta\tilde\bx_{k,l})
		\right\rvert\,\mathrm d\nu\mathrm d\tau\\
		+\int_0^1\int_0^\tau\left\lvert
		\left[
		D^2f(\bx_{k}+\nu\delta\bx_{k,l})-D^2f(\tilde\bx_{k}+\nu\delta\tilde\bx_{k,l})\right](\delta\tilde\bx_{k,l},\delta\tilde\bx_{k,l})\right\rvert\,\mathrm
		d\nu\mathrm d\tau.
	\end{multlined}
	\end{align*}
	We now note that by symmetry of \(D^2f(\bx)\), it holds that for any \(\mathbf a,\mathbf b\in\R^n\) we have the identity
	\[
		D^2f(\bx)(\mathbf a,\mathbf a)-D^2f(\bx)(\mathbf b,\mathbf b)=D^2f(\bx)(\mathbf a-\mathbf b,\mathbf a+\mathbf b).
	\]
	Hence, the first term may be bounded by
	\[
		\|f\|_{\cC^3_\b}2^{1-1/p}\left\{\|\bx-\tilde\bx\|^{p/2}_{p;[k,l]}(\|\bx\|^p_{p;[k,l]}+\|\tilde\bx\|^p_{p;[k,l]})^{1/2}\right\}^{2/p}.
	\]
	The second term can be bounded, as before, by
	\[
		\|f\|_{\cC^3_\b}\left( \|\tilde\bx\|_{p;[k,l]}^p\|\bx-\tilde\bx\|^{p/2}_{\infty;[0,l]}\right)^{2/p}.
	\]
	Putting both terms together and proceeding as before we obtain the bound
	\[
		\|T-\tilde T\|_{p/2;[k,l]}\le 2^{1-2/p}\|f\|_{\cC^3_\b}\left[
		\|\bx-\tilde\bx\|^{p/2}_{p;[k,l]}(\|\bx\|^p_{p;[k,l]}+\|\tilde\bx\|^p_{p;[k,l]})^{1/2}+\|\tilde\bx\|_{p;[k,l]}^p\|\bx-\tilde\bx\|^{p/2}_{\infty;[0,l]}
	\right]^{2/p}.\qedhere
	\]
\end{proof}

\subsection{The Young regime}
In this regime, we can easily obtain good bounds with minimal assumptions on the \(f_i\).
These bounds have already been shown by Davie \cite{Dav08}, but it will be an enlightening
exercise to go through the proof in full details, since it will lay the foundations for
our approach in the rough regime.
Also, our methods are slightly different and already in this case they highlight the
importance of the rôle played by the Sewing Lemma (\Cref{proposition:sew,proposition:gensew}) and the rough Grönwall lemma (\Cref{prop:gronwall}).

Before beginning we define the remainder
\begin{equation}
    R_{k,l}\coloneq \bx_{k,l}-\sum_{\mu=1}^df_\mu(\bx_k)\bw^\mu_{k,l}
    \label{eq:young.remainder}
\end{equation}
so that
\[ \bx_{k,l}=\sum_{\mu=1}^df_\mu(\bx_k)\bw^\mu_{k,l}+R_{k,l}. \]

\begin{theorem}
Let \(1\le p<2\), and suppose that \(f=(f_1,\dotsc,f_d)\) is a collection of vector fields  in
\(\R^n\), of class \(\cC^1_{\mathrm b}\).
  The bound
  \begin{equation}
\label{eq:resnet.bound.1}
\|\bx\|_{p;[k,l]}\le 2\left(2^pC_{p,N}^{p-1}\|f\|_{\cC^1_\b}^p\|\bw\|_{p;[k,l]}^p\vee 2\|f\|_{\cC^1_\b}\|\bw\|_{p;[k,l]}\right)
  \end{equation}
  holds, with
	\[C_{p,N}\coloneq 2^{2/p}\zeta_N(2/p). \]
  \label{thm:young.regime}
\end{theorem}
\begin{proof}
  Consider the triangular array
  \(\Xi_{k,l}\coloneq \sum_\mu f_\mu(\bx_k)\bw^\mu_{k,l}\).
  By \cref{eq:deltader} we immediately see that
	\[\delta\Xi_{k,l,m}=-\sum_{\mu=1}^d(f_\mu(\bx_l)-f_\mu(\bx_k))\bw^\mu_{l,m}, \]
  so that the usual Lipschitz bound implies
  \begin{equation*}
    \lvert\delta\Xi_{k,l,m}\rvert\le\|f\|_{\cC^1_\b}\lVert\bx\rVert_{p;[k,l]}\|\bw\|_{p;[k,l]},
  \end{equation*}
  and the hypothesis of \Cref{proposition:sew} is satisfied since \(2/p>1\).
  Thus, we obtain
  \[
    \left\lvert\sum_{j=k}^{l-1}\Xi_{j,j+1}-\Xi_{k,l}\right\rvert\le C_{p,N}
    \|f\|_{\cC^1_\b}\lVert\bx\rVert_{p;[k,l]}\|\bw\|_{p;[k,l]}.
  \]
	with \(C_{p,N}\coloneq 2^{2/p}\zeta_N(2/p)\).
  Now, we observe that by \cref{eq:cre},
  \[ \sum_{j=k}^{l-1}\Xi_{j,j+1}=\bx_{k,l} \]
  thus obtaining
  \begin{equation}
    \lvert R_{k,l}\rvert\le C_{p,N}\lVert f\rVert_{\cC^1_\b}
    \lVert\bx\rVert_{p;[k,l]}\|\bw\|_{p;[k,l]}.
    \label{eq:L1.R.bound}
  \end{equation}
  By \Cref{lem:pvar.control}, the same bound holds if we replace \(\lvert R_{k,l}\rvert\) on
  the left-hand side by \(\lVert R\rVert_{p/2;[k,l]}\).

  Using the relation between the remainder \(R\) and the increments of \(\bx\) we get
  \[
    \lvert\bx_{k,l}\rvert\le C_{p,N}\lVert f\rVert_{\cC^1_\b}\lVert\bx\rVert_{p;[k,l]}
    \|\bw\|_{p;[k,l]}+\lVert f\rVert_{\cC^{1}_\b}\|\bw\|_{p;[k,l]}
  \]
  for all \(0\le l<k\le N\).
  We deduce that
  \[ \lVert\bx\rVert^p_{p;[k,l]}\le
      2^{p-1}C^p_p\|f\|^p_{\cC^1_\b}\lVert\bx\rVert^p_{p;[k,l]}\|\bw\|_{p;[k,l]}^p
      +2^{p-1}\|f\|^p_{\cC^1_\b}\|\bw\|_{p;[k,l]}^p.
  \]

  If we now consider a pair \(k<l\) such that \(\bar \omega_{k,l}^{1/p}\coloneq 2C_{p,N}\|f\|_{\cC^1_\b}\|\bw\|_{p;[k,l]}\le 1\), we obtain
  \[
      \lVert\bx\rVert^p_{p;[k,l]}\le 2^p\|f\|_{\cC^1}^p\|\bw\|_{p;[k,l]}^p=C_{p,N}^{-p}\bar \omega_{k,l}
  \]
  for all such \((k,l)\). In particular
  \[
      |\bx_{k,l}|\le C_{p,N}^{-1}\bar \omega_{k,l}^{1/p}.
  \]
  By \cref{eq:cre} the same inequality also holds when \(l=k+1\).
  From \Cref{lem:maxbound} we then get
  \begin{align*}
      \|\bx\|_{p;[k,l]}&\le 3C_{p,N}^{-1}\left( \bar \omega_{k,l}\vee\bar \omega_{k,l}^{1/p}\right)\\
      &= 3C_{p}^{-1}\left(2^p\|f\|^p_{\cC^1_\b}C_{p,N}^p\|\bw\|_{p;[k,l]}^p\vee 2\|f\|_{\cC^1_\b}C_{p,N}\|\bw\|_{p;[k,l]}\right)
  \end{align*}
  from where the result follows.
\end{proof}
\begin{remark}
	The hypothesis on the vector fields \(f\), namely \(f\in\cC^1_\b\), can be relaxed to \(f\in\mathrm{Lip}^{\gamma-1}\) for some \(\gamma\in(p,2]\),
	meaning that \(f\) need not be differentiable but we merely need the existence of positive constant \(L\) such that
	\[
		|f(\bx)-f(\tilde\bx)|\le L|\bx-\tilde\bx|^{\gamma-1}
	\]
	for all \(\bx,\tilde\bx\in\R^n\).
\end{remark}

Finally we show that
\begin{theorem}
	\label{thm:main.young}
    Let \(1\le p<2\) and suppose \(\bx,\tilde{\bx}\) are two solutions to \cref{eq:cre} with initial
		conditions \(\xi,\tilde{\xi}\) and driven by \(\bw,\tilde{\bw}\) respectively.
		If furthermore \(f_1,\dotsc,f_d\in\cC^2_\b\) are such that \(\max_{\mu=1,\dotsc,d}\|f_\mu\|_{\cC^2_\b}\le L\), then
    \[
			\sup_{k=0,\dotsc,N}|\bx_k-\tilde{\bx}_k|\le 2c_{p,N}^{1/p}e^{c_{p,N} L^p(\|\bw\|^p_{p;[0,N]}+\|\tilde\bw\|_{p;[0,N]}^p)}(|\xi-\tilde{\xi}|+L\|\bw-\tilde\bw\|_{p;[0,N]})
    \]
		holds, where 
		\[
			c_{p,N}\coloneq (4e^2)^p(4^{p-1}C_{p,N}^p+1)
		\]
		and \(C_{p,N}\) is as in \Cref{thm:young.regime}.
\end{theorem}
\begin{proof}
	In order to make the notation more compact we also define the controls
	\[
		\varepsilon_{k,l}\coloneq\|\bw-\tilde\bw\|_{p;[k,l]}^p,\quad\omega_{k,l}\coloneq\|\bw\|^p_{p;[k,l]}+\|\tilde\bw\|^p_{p;[k,l]}.
	\]

	Now, we define \(\bz_k\coloneq\bx_k-\tilde\bx_k\) and notice that
    \[
        |\bz_{k,l}|\le|R_{k,l}-\tilde
        R_{k,l}|+\sum_{\mu=1}^d|f_\mu(\bx_k)\bw_{k,l}-f_\mu(\tilde\bx_k)\tilde\bw_{k,l}|.
    \]

    For the second term we have the bound
		\begin{linenomath}
			\begin{align*}
				\sum_{\mu=1}^d|f_\mu(\bx_k)\bw^\mu_{k,l}-f_\mu(\tilde\bx_k)\tilde\bw^\mu_{k,l}|&\le\|f\|_{\infty}\varepsilon_{k,l}^{1/p}+\|Df\|_{\infty}|\bz_k|\|\tilde\bw\|_{p;[k,l]}\\
				&\le L\left(\varepsilon_{k,l}^{1/p}+|\bz_k|\|\tilde\bw\|_{p;[k,l]}\right).
			\end{align*}
		\end{linenomath}

    To bound the first term, we use the Sewing Lemma with the germ \(
    \Xi_{k,l}\coloneq\sum_\mu f_\mu(\bx_k)\bw^\mu_{k,l}-\sum_\mu f_\mu(\tilde\bx_k)\tilde\bw^\mu_{k,l}\).
    First we compute
    \[
        \delta\Xi_{k,l,m}=-\sum_{\mu=1}^d(f_\mu(\bx_l)-f_\mu(\bx_k))\bw^\mu_{l,m}-\sum_{\mu=1}^d(f_\mu(\tilde\bx_l)-f_\mu(\tilde\bx_k))\tilde\bw^\mu_{l,m}
    \]
    so that
		\begin{align*}
			|\delta\Xi_{k,l,m}|&\le \|f\|_{\cC^1_\b}\|\bx\|_{p;[k,l]}\varepsilon_{l,m}^{1/p}+\sum_{i=1}^d\|f_\mu(\bx)-f_\mu(\tilde\bx)\|_{p;[k,l]}\|\tilde\bw^\mu\|_{p;[l,m]}\\
			&\le L\|\bx\|_{p;[k,l]}\varepsilon_{l,m}^{1/p}+\max_{\mu=1,\dotsc,d}\|f_\mu(\bx)-f_\mu(\tilde\bx)\|_{p;[k,l]}\|\tilde\bw\|_{p;[l,m]}
		\end{align*}

		Hence by \Cref{lem:f.bound}
		\begin{linenomath}
			\begin{align*}
				\left|\sum_{j=k}^{l-1}\Xi_{j,j+1}-\Xi_{k,l}\right|&= |R_{k,l}-\tilde R_{k,l}|\\
				&\le
				2^{1-1/p}LC_{p,N}\left( \|\bx\|_{p;[k,l]}\varepsilon_{k,l}^{1/p}+\Bigl( \|\bz\|_{\infty;[k,l]}^p\|\tilde\bx\|_{p;[k,l]}^p+\|\bz\|^p_{p;[k,l]} \Bigr)^{1/p}\|\tilde\bw\|_{p;[k,l]} \right).
			\end{align*}
		\end{linenomath}

    Now, on the one hand, we see that
		\begin{linenomath}
			\begin{align*}
				|\bz_{k,l}|&\le |R_{k,l}-\tilde R_{k,l}|+\sum_{\mu=1}^d|f_\mu(\bx_k)\bw^\mu_{k,l}-f_\mu(\tilde\bx_k)\tilde\bw^\mu_{k,l}|\\
				&\le |R_{k,l}-\tilde R_{k,l}| + L\left(\|\bz\|_{\infty;[k,l]}\|\tilde\bw\|_{p;[k,l]} + \varepsilon_{k,l}\right),
			\end{align*}
		\end{linenomath}
    so the bound
		\[ \|\bz\|_{p;[k,l]}\le 8^{1-1/p}LC_{p,N}\left\{ \|\bx\|_{p;[k,l]}^p\varepsilon_{k,l}+\left( \|\bz\|^p_{\infty;[k,l]}\|\tilde\bx\|^p_{p;[k,l]}+\|\bz\|^p_{p;[k,l]} \right)\|\tilde\bw\|_{p;[k,l]}^p + \|\bz\|_{\infty;[0,l]}^p\|\tilde\bw\|_{p;[k,l]}^{p}+\varepsilon_{k,l} \right\}^{1/p}
		\]
    holds.
		Therefore, for any pair of indices \(k<l\) such that \(8^{p-1}L^pC_{p,N}^p\omega_{k,l}\le\frac12\), we have that
		\[
			\|\bz\|_{p;[k,l]}^p\le 2^{3p-2}L^pC_{p,N}^p\left(1+\|\bx\|_{p;[k,l]}^p\right)\varepsilon_{k,l}+\|\bz\|_{\infty;[k,l]}^p(2^{3p-2}L^pC_{p,N}^p\|\tilde\bw\|_{p;[k,l]}^p+\|\tilde\bx\|^p_{p;[k,l]}).
		\]
		By the a priori estimate in \Cref{thm:young.regime}, we see that
		\[
			\|\bx\|_{p;[k,l]}\le 2L\|\bw\|_{p;[k,l]},\quad \|\tilde\bx\|_{p;[k,l]}\le 2L\|\tilde\bw\|_{p;[k,l]}
		\]
		so that
		\[
			\|\bz\|_{p;[k,l]}^p\le A_p\varepsilon_{k,l}+A_p\|\bz\|_{\infty;[0,l]}^p\omega_{k,l}.
		\]
		where \(A_p\coloneq 2^pL^p(4^{p-1}C_{p,N}^p+1)\).

    On the other hand, when \(l=k+1\) we have
		\begin{linenomath}
			\begin{align*}
				|\bz_{k+1}-\bz_k|&\le\sum_{\mu=1}^d\Bigl| f_\mu(\bx_k)\bw^\mu_{k,k+1}-f_\mu(\tilde\bx_k)\tilde\bw^\mu_{k,k+1}
				\Bigr|\\
				&\le A_p^{1/p}\varepsilon_{k,l}^{1/p}+A_p^{1/p}\|\bz\|_{\infty;[0,k+1]}\|\tilde\bw\|_{p;[k,l]}
			\end{align*}
		\end{linenomath}

		Finally, by using \Cref{prop:gronwall} we obtain
		\begin{linenomath}
			\begin{equation*}
				|\bx_N-\tilde\bx_N|\le 2c_{p,N}e^{c_{p,N}\|\tilde\bw\|_{p;[0,N]}^p}\left(|\xi-\tilde\xi|+\|\bw-\tilde\bw\|_{p;[0,N]} \right)
			\end{equation*}
		\end{linenomath}
		where we have used that \(x\mapsto (1+cx)^\alpha e^{-x}\) is decreasing over \([0,\infty)\) as long as \(c\alpha\le 1\), and
		\[
			c_{p,N}\coloneq 2^pe^{2p}A_p.\qedhere
		\]
\end{proof}
\begin{remark}
	As before, the hypothesis on the vector fields can be relaxed to requiring that \(f\in\mathrm{Lip}^\gamma\) for some \(\gamma\in(p,2]\).
	In this case this means that \(f\in\cC^1_\b\) and there is a constant \(L>0\) such that
	\[
		\|Df(\bx)-Df(\tilde\bx)\|\le L|\bx-\tilde\bx|^{\gamma-1}
	\]
	for all \(\bx,\tilde\bx\in\R^n\).
\end{remark}

\subsection{The case of \texorpdfstring{\(2\le p<3\)}{2≤p<3}}
We show analogues of the results in the previous section for the case where now we take \(p\in[2,3)\).

We keep the previous notations, i.e., we consider \cref{eq:cre} and but redefine \(R\) in \cref{eq:young.remainder} as 
\begin{equation}
	R_{k,l}\coloneq	\bx_{k,l}-\sum_{\mu=1}^df_\mu(\bx_k)\bw_{k,l}^\mu-\sum_{\mu,\nu=1}^dDf_\nu(\bx_k)f_\mu(\bx_k)\bbW_{k,l}^{\mu\nu},
	\label{eq:rough.remainder}
\end{equation}
and we furthermore consider
\begin{align}
	I_{k,l}&\coloneq \bx_{k,l}-\sum_{\mu=1}^df_\mu(\bx_k)\bw_{k,l}^\mu\label{eq:I}\\
	J^\mu_{k,l}&\coloneq f_\mu(\bx_l)-f_\mu(\bx_k)-\sum_{\nu=1}^dDf_\mu(\bx_k)f_\nu(\bx_k)\bw_{k,l}^{\nu}\label{eq:J}\\
	&= f_\mu(\bx_l)-f_\mu(\bx_k)-Df_\mu(\bx_k)\delta\bx_{k,l}+Df_\mu(\bx_k)I_{k,l}.\label{eq:jdef.2}
\end{align}
where in \cref{eq:rough.remainder,eq:J}, \(\bbW\) denotes the iterated-sums lift of \(\bw\).
\begin{definition}
	For \(\mu,\nu\in\{1,\dotsc,d\}\) we define the vector field \(F_{\mu\nu}\colon\R^n\to\R^n\)
	\[
		F_{\mu\nu}(\bx)\coloneqq Df_\nu(\bx)f_\mu(\bx).
	\]
\end{definition}
Observe that by successive application of the chain rule one can show that if
\[
	\|F_{\mu\nu}\|_{C^k_\b}\le (2^{k+1}-1)\|f\|^2_{C^{k+1}_\b}
\]
for all \(k\ge 0\) as long as the norm on the right-hand side is finite.

\begin{lemma}\label{lem:J.bound}
	Let \(p\in[2,3)\) and \(f\in\cC^2_\b\). The bound
	\[
		\max_{\mu=1,\dotsc,d}\|J^\mu\|_{p/2;[k,l]}\le 2^{1-2/p}\|f\|_{\cC^2_\b}\left( \|I\|^{p/2}_{p/2;[k,l]}+\frac12\|\bx\|_{p;[k,l]}^{p} \right)^{2/p}.
	\]
	holds.
\end{lemma}
\begin{proof}
	Performing a first-order Taylor expansion on \(f_i\) we see that
	\[
		J^\mu_{k,l}=Df_\mu(\bx_k)\left( \bx_{k,l}-\sum_{\nu=1}^df_\nu(\bx_k)\bw^\nu_{k,l} \right)+\frac12 D^2f_\mu(\bx_k+\theta\bx_{k,l})(\bx_{k,l},\bx_{k,l})
	\]
	for some \(\theta\in(0,1)\).
	Thus
	\begin{align*}
		|J^\mu_{k,l}|&\le\|f\|_{\cC^2_\b}\left( |I_{k,l}|+\frac{1}{2}|\bx_{k,l}|^2 \right)\\
		&\le 2^{1-p/2}\|f\|_{\cC^2_\b}\left( \|I\|_{p/2;[k,l]}^{p/2}+\frac12\|\bx\|^p_{p;[k,l]} \right)^{2/p}.
	\end{align*}
	The proof is concluded by applying \Cref{lem:pvar.control}.
\end{proof}

\begin{theorem}
	\label{thm:rough.apriori}
	Let \(p\in[2,3)\), and suppose that \(\bx\) solves \cref{eq:cre} with \(f\in C^2_\b\).
	The bounds
	\begin{align*}
		\|\bx\|_{p;[k,l]}&\le K_p(\t|\bW|_{p;[k,l]}^p\vee\t|\bW|_{p;[k,l]})\\
		\|I\|_{p/2;[k,l]}&\le K'_p(\t|\bW|_{p;[k,l]}^{2p}\vee\t|\bW|_{p;[k,l]}^2)
	\end{align*}
	hold, with
	\[
		K_p\coloneq 9\times 2^{6(1-1/p)}\left( 1\vee 6^{1-1/p}8^{(1-1/p)(1-2/p)}C_{p,N}^{1-1/p} \right),\quad K'_p\coloneq 3\times 2^{1-2/p}\left( 1+K_p^2 \right).
	\]
\end{theorem}
\begin{proof}
	As before, by scaling we may assume that \(\|f\|_{C^2_\b}\le 1\).
	Consider the triangular array
	\[
		\Xi_{k,l}\coloneq\sum_{\mu=1}^df_\mu(\bx_k)\bw_{k,l}^\mu+\sum_{\mu,\nu=1}^d F_{\mu\nu}(\bx_k)\bbW_{k,l}^{\mu\nu}.
	\]
	We immediately see that
	\begin{align*}
		\delta\Xi_{k,l,m}&= -\sum_{\mu=1}^d\Bigl(f_\mu(\bx_l)-f_\mu(\bx_k)\Bigr)\bw_{l,m}^\mu+\sum_{\mu,\nu=1}^d\Bigl\{
		F_{\mu\nu}(\bx_k)\bw_{k,l}^\mu\bw_{l,m}^\nu-\bigl(F_{\mu\nu}(\bx_l)-F_{\mu\nu}(\bx_k)\bigr)\bbW_{l,m}^{\mu\nu} \Bigr\}\\
		&= -\sum_{\mu=1}^dJ_{k,l}^\mu\bw_{k,l}^\mu-\sum_{\mu,\nu=1}^d\bigl(F_{\mu\nu}(\bx_l)-F_{\mu\nu}(\bx_k)\bigr)\bbW_{l,m}^{\mu\nu}.
	\end{align*}
	
	Since \(f\in C^2_\b\), the function \(F_{\mu\nu}\) is in \(C^1_\b\) for all \(\mu,\nu\in\{1,\dotsc,d\}\) and \(\|F_{\mu\nu}\|_{C^1_\b}\le 3\).
	Therefore, we have the \(p\)-variation estimate
	\[
		\|F_{\mu\nu}(\bx)\|_{p;[k,l]}\le 3\|\bx\|_{p;[k,l]}.
	\]
	Hence, we see that
	\[
		|\delta\Xi_{k,l,m}|\le \sum_{\mu=1}^d\|J^\mu\|_{p/2;[k,l]}\|\bw^\mu\|_{p;[l,m]}+3\|\bx\|_{p;[k,l]}\sum_{\mu,\nu=1}^d\|\bbW^{\mu\nu}\|_{p/2;[l,m]}.
	\]
	By \Cref{proposition:gensew} we see that
	\begin{equation}
		\label{eq:R.bound}
		|R_{k,l}|\le 3C_{p,N}\left( \sum_{\mu=1}^d\|J^\mu\|_{p/2;[k,l]}\|\bw^\mu\|_{p;[k,l]}+\|\bx\|_{p;[k,l]}\|\bbW\|_{p/2;[k,l]} \right).
	\end{equation}
	
	Now we note that
	\[
		|I_{k,l}|\le|R_{k,l}|+\sum_{\mu,\nu=1}^d\left| F_{\mu\nu}(\bx_k)\bbW_{k,l}^{\mu\nu} \right|\le|R_{k,l}|+\|\bbW\|_{p/2;[k,l]}
	\]
	so that, by \cref{eq:R.bound} and \Cref{lem:J.bound}, we obtain
	\begin{align*}
		|I_{k,l}|&\le 3\times2^{1-2/p}C_{p,N}\left\{(\|I\|_{p/2;[k,l]}^{p/2}+\|\bx\|_{p;[k,l]}^p)^{2/p}\|\bw\|_{p;[k,l]}+\|\bx\|_{p;[k,l]}\|\bbW\|_{p/2;[k,l]}
		\right\}+\|\bbW\|_{p/2;[k,l]}\\
		&\le 3\times
		4^{1-2/p}C_{p,N}\left\{(\|I\|_{p/2;[k,l]}^{p/2}+\|\bx\|_{p;[k,l]}^p)\|\bw\|^{p/2}_{p;[k,l]}+\|\bx\|^{p/2}_{p;[k,l]}\|\bbW\|^{p/2}_{p/2;[k,l]}
		\right\}^{2/p}+\|\bbW\|_{p/2;[k,l]}.
	\end{align*}
	Taking \(\tfrac{p}{2}\)-variation we obtain
	\[
		\|I\|_{p/2;[k,l]}\le 3\times
		8^{1-2/p}C_{p,N}\left\{(\|I\|_{p/2;[k,l]}^{p/2}+\|\bx\|_{p;[k,l]}^p)\|\bw\|^{p/2}_{p;[k,l]}+\|\bx\|^{p/2}_{p;[k,l]}\|\bbW\|^{p/2}_{p/2;[k,l]}
		\right\}^{2/p}+2^{1-2/p}\|\bbW\|_{p/2;[k,l]}.
	\]
	
	If \(0\le k<l\le N\) are such that \(3\times 8^{1-2/p}C_{p,N}\t|\bW|_{p;[k,l]}\le\frac12\) then
	\begin{equation}
		\label{eq:I.bound}
		\|I\|_{p/2;[k,l]}\le 3\times 2^{1-2/p}(\|\bx\|_{p;[k,l]}^2+\|\bbW\|_{p/2;[k,l]}).
	\end{equation}

	Finally, noting that
	\[
		|\bx_{k,l}|\le|I_{k,l}|+\sum_{\mu=1}^d|f_\mu(\bx_k)\bw^\mu_{k,l}|\le\|I\|_{p/2;[k,l]}+\|\bw\|_{p;[k,l]}.
	\]
	we obtain, by taking \(p\)-variation, that
	\begin{align*}
		\|\bx\|_{p;[k,l]}&\le 2^{1-1/p}(\|I\|_{p/2;[k,l]}+\|\bw\|_{p;[k,l]})\\
		&\le 3\times 2^{2-3/p}\|\bx\|_{p;[k,l]}^2+3\times 2^{2-3/p}\|\bbW\|_{p/2;[k,l]}+2^{1-1/p}\|\bw\|_{p;[k,l]}.
	\end{align*}
	Let \(c_1\coloneq 3\times 2^{2-3/p}\), \(c_2\coloneq 2^{1-1/p}\).
	Multiplying both sides by \(c_1\) and using our hypothesis on the interval \([k,l]\) we obtain that
	\[
		c_1\|\bx\|_{p;[k,l]}\le(c_1\|\bx\|_{p;[k,l]})^2+c_1(c_1+c_2)\t|\bW|_{p;[k,l]}.
	\]
	Set \(c\coloneq c_1(c_1+c_2)\).
	Reducing further the size of the interval if necessary, we may assume that \(c_1(c_1+c_2)\t|\bW|_{p;[k,l]}\le\frac14\), so that we must necessarily have that one of the following inequalities hold:
	\[
		4\|\bx\|_{p;[k,l]}\ge\frac{1+\sqrt{1-4c\t|\bW|_{p;[k,l]}}}{2}\ge\frac12,\quad 4\|\bx\|_{p;[k,l]}\le\frac{1-\sqrt{1-4c\t|\bW|_{p;[k,l]}}}{2}\le 2c\t|\bW|_{p;[k,l]}.
	\]
	In fact, the second inequality holds if \(\|\bx\|_{p;[k,l]}\le\frac18\).
	Applying \Cref{lem:maxbound}, we obtain
	\[
		\|\bx\|_{p;[k,l]}\le K_p\left( \t|\bW|_{p;[k,l]}\vee\t|\bW|_{p;[k,l]}^p \right)
	\]
	with
	\[
		K_p\coloneq 9\times 2^{6(1-1/p)}\left( 1\vee 6^{1-1/p}8^{(1-1/p)(1-2/p)}C_{p,N}^{1-1/p} \right).
	\]
	This shows the first estimate.

	Now replace this bound in \cref{eq:I.bound} and use the fact that \(\|\bbW\|_{p/2;[k,l]}\le\t|\bW|_{p;[k,l]}^2\) to obtain
	\[
		\|I\|_{p/2;[k,l]}\le 3\times2^{1-2/p}\left(1+ K_p^2\right)\t|\bW|_{p;[k,l]}^2
	\]
\end{proof}

Finally, we prove our main result, namely the stability bound for the evolution of features through the network.
But first, we extend \Cref{lem:J.bound} to bound the difference of the remainders \(J\) and \(\tilde J\) for solutions of difference equations driven
by different noises.
\begin{lemma}
	Let \(\bx\) and \(\tilde\bx\) be solutions to \cref{eq:cre} driven by \(\bw\) and \(\tilde\bw\), respectively.
	Then, for all \(0\le k<l\le N\) we have
	\[
		\max_{\mu=1,\dotsc,d}\|J^\mu-\tilde J^\mu\|_{p/2;[k,l]}\le\begin{multlined}[t]2^{2-4/p}\|f\|_{\cC^3_\b}\Bigl\{\|I-\tilde I\|_{p/2;[k,l]}+\|\bx-\tilde\bx\|_{\infty;[0,l]}\left( \|\tilde
		I\|_{p/2;[k,l]}+\|\tilde\bx\|_{p;[k,l]}^2 \right)\\
+\|\bx-\tilde\bx\|_{p;[k,l]}(\|\bx\|_{p;[k,l]}+\|\tilde\bx\|_{p;[k,l]})\Bigr\}\end{multlined}
	\]
	\label{lem:J.bound.diff}
\end{lemma}
\begin{proof}
	Using \cref{eq:jdef.2} we see that
	\[
		J^\mu_{k,l}-\tilde J^\mu_{k,l}=T^\mu_{k,l}-\tilde T^\mu_{k,l}+B_{k,l}
	\]
	where,
	\[ 
		\begin{aligned}
			B_{k,l}&\coloneq Df_\mu(\bx_k)I_{k,l}-Df_\mu(\tilde\bx_k)\tilde I_{k,l},\\
			T^\mu_{k,l}&\coloneq f_\mu(\bx_l)-f_\mu(\bx_l)-Df_\mu(\bx_k)\delta\bx_{k,l},
		\end{aligned}
	\]
	and \(\tilde T^i\) is defined similarly.

	Adding and subtracting cross terms we obtain the following bound for the \(B\) term:
	\[
		|B_{k,l}|\le\|f\|_{\cC^3_\b}(|I_{k,l}-\tilde I_{k,l}|+|\bx_k-\tilde\bx_k||\tilde I_{k,l}|),
	\]
	so that
	\[
		\|B\|_{p/2;[k,l]}\le 2^{1-2/p}\|f\|_{\cC^3_\b}\left( \|I-\tilde I\|^{p/2}_{p/2;[k,l]}+\|\bx-\tilde\bx\|_{\infty;[0,l]}^{p/2}\|\tilde
		I\|_{p/2;[k,l]}^{p/2} \right)^{2/p}.
	\]
	From \Cref{lem:f.bound} we obtain that
	\[
		\|T-\tilde T\|_{p/2;[k,l]}\le2^{1-2/p}\|f\|_{\cC^3_\b}\left[
		\|\bx-\tilde\bx\|_{p;[k,l]}(\|\bx\|_{p;[k,l]}+\|\tilde\bx\|_{p;[k,l]})+\|\tilde\bx\|_{p;[k,l]}^2\|\bx-\tilde\bx\|_{\infty;[0,l]} \right]
	\]
	and the proof is finished.
\end{proof}
\begin{theorem}
    Let \(2\le p<3\) and suppose \(\bx,\tilde{\bx}\) are two solutions to \cref{eq:cre} with initial
		conditions \(\xi,\tilde{\xi}\) and driven by \(\bw,\tilde{\bw}\) respectively.
    If furthermore \(f_1,\dotsc,f_d\in\cC^3_\b\), then
    \[
			\sup_{k=0,\dotsc,N}|\bx_k-\tilde{\bx}_k|\le
			2c_{p,N}'e^{c_{p,N}\|f\|_{\cC^3_\b}^p\left(\t|\bW|^p_{p;[0,N]}+\t|\tilde\bW|_{p;[0,N]}^p\right)}(|\xi-\tilde{\xi}|+\|f\|_{\cC^3_\b}\rho_p(\bW,\tilde\bW))
    \]
		holds, where 
		\[
			c_{p,N}\coloneq 2^pe^{2p}(L_p+K_p^2+K_p')^p,\quad c_{p,N}'=2^{1-2/p}c_{p,N}^{1/p}
		\]
		with
		\[
			L_p\coloneq 4^{3/2-2/p}\times 7^{2-3/p}\times C_{p,N},
		\]
		the constant \(C_{p,N}\) appears in \Cref{proposition:sew} and \(K_p,K_p'\) are as in \Cref{thm:rough.apriori}.
	\label{thm:main.resnet}
\end{theorem}
\begin{proof}
	We divide the proof in several steps.
	Below we denote
	\begin{align*}
		\Delta\bx_k&\coloneq\bx_k-\tilde\bx_k\\
		\Delta J^\mu_{k,l}&\coloneq J^\mu_{k,l}-\tilde J^\mu_{k,l}\\
		\Delta I_{k,l}&\coloneq I_{k,l}-\tilde I_{k,l}
	\end{align*}
	and we consider the controls
	\begin{align*}
		\varepsilon_{k,l}&\coloneq \|\bw-\tilde\bw\|^p_{p;[k,l]}\\
		\omega_{k,l}&\coloneq \|\bx\|_{p;[k,l]}^p+\|\tilde\bx\|_{p;[k,l]}^p\\
		E_{k,l}&\coloneq\|\bbW-\tilde\bbW\|^{p/2}_{p/2;[k,l]}.
	\end{align*}
	We also assume, without loss of generality, that \(\|f\|_{\cC^3_\b}\le 1\).
	\begin{enumerate}[label=Step \arabic*., font=\bfseries]
		\item We estimate the difference of the remainders \(R\) and \(\tilde R\) as defined in \cref{eq:rough.remainder} via the Sewing Lemma.
			To this end, consider the germ
			\[
				\Xi_{k,l}\coloneq \sum_{\mu=1}^df_\mu(\bx_k)\bw^\mu_{k,l}+\sum_{\mu,\nu=1}^d
				F_{\mu\nu}(\bx_k)\bbW^{\mu\nu}_{k,l}-\sum_{\mu=1}^df_\mu(\tilde\bx_k)\tilde\bw^\mu_{k,l}-\sum_{\mu,\nu=1}^d
				F_{\mu\nu}(\bx_k)\tilde\bbW^{\mu\nu}_{k,l}.
			\]
			A standard calculation, using Chen's identity \Cref{thm:chen} yields
			\begin{align*}
				\delta\Xi_{k,l,m}&= \begin{multlined}[t]-\sum_{\mu=1}^d(f_\mu(\bx_l)-f_\mu(\bx_k))\bw^\mu_{l,m}+\sum_{\mu,\nu=1}^d\left(
					F_{\mu\nu}(\bx_k)\bw^\mu_{k,l}\bw^\nu_{l,m}- (F_{\mu\nu}(\bx_l)-F_{\mu\nu}(\bx_k))\bbW_{l,m}^{\mu\nu} \right)\\
					+\sum_{\mu=1}^d( f_\mu(\tilde\bx_l)-f_\mu(\tilde\bx_k) )\tilde\bw_{l,m}^\mu-\sum_{\mu,\nu=1}^d\left(
					F_{\mu\nu}(\tilde\bx_k)\tilde\bw^\mu_{k,l}\tilde\bw^\nu_{l,m}-(F_{\mu\nu}(\tilde\bx_l)-F_{\mu\nu}(\tilde\bx_k))\tilde\bbW^{\mu\nu}_{l,m} \right)
				\end{multlined}\\
				&= -\sum_{\mu=1}^dJ^\mu_{k,l}\bw^\mu_{l,m}-\sum_{\mu,\nu=1}^d(F_{\mu\nu}(\bx_l)-F_{\mu\nu}(\bx_k))\bbW^{\mu\nu}_{l,m}+\sum_{\mu=1}^d\tilde
				J^\mu_{k,l}\tilde\bw^\mu_{l,m}+\sum_{\mu,\nu=1}^d(F_{\mu\nu}(\tilde\bx_l)-F_{\mu\nu}(\tilde\bx_k))\tilde\bbW^{\mu\nu}_{l,m}.
			\end{align*}
			Therefore
			\begin{equation*}
				|\delta\Xi_{k,l,m}|\le \begin{multlined}[t]
					\sum_{\mu=1}^d\|\Delta J^\mu\|_{p/2;[k,l]}\|\bw^\mu\|_{p;[l,m]}+\sum_{i=1}^d\|\tilde J^\mu\|_{p;[k,l]}\|\Delta\bw^\mu\|_{p;[l,m]}\\
					+\sum_{\mu,\nu=1}^d\|F_{\mu\nu}(\bx)-F_{\mu\nu}(\tilde\bx)\|_{p;[k,l]}\|\bbW^{\mu\nu}\|_{p/2;[l,m]}+\sum_{\mu,\nu=1}^d\|F_{\mu\nu}(\tilde\bx)\|_{p;[k,l]}\|\bbW^{\mu\nu}-\tilde\bbW^{\mu\nu}\|_{p/2;[l,m]}.
				\end{multlined}
			\end{equation*}
			Hence, by the Sewing Lemma we obtain that
			\begin{equation}
				\label{eq:R.diff.bound}
				|R_{k,l}-\tilde R_{k,l}|\le C_{p,N}\Bigl\{\|\Delta J\|_{p/2;[k,l]}\|\bw\|_{p;[k,l]}+\|\tilde J\|_{p/2;[k,l]}\varepsilon_{k,l}^{1/p}+\|\Delta F\|_{p;[k,l]}\|\bbW\|_{p/2;[k,l]}+\|F(\tilde\bx)\|_{p;[k,l]}E_{k,l}^{2/p}\Bigr\}.
			\end{equation}

		\item We now use the relation \(I_{k,l}=R_{k,l}+\sum_{\mu,\nu=1}^dF_{\mu\nu}(\bx_k)\bbW^{\mu\nu}_{k,l}\) to obtain
			\begin{align*}
				|I_{k,l}-\tilde I_{k,l}|&\le|R_{k,l}-\tilde
				R_{k,l}|+\sum_{\mu,\nu}^d|F_{\mu\nu}(\bx_k)-F_{\mu\nu}(\tilde\bx_k)||\bbW^{\mu\nu}_{k,l}|+\sum_{\mu,\nu=1}^d|F_{\mu\nu}(\tilde\bx_k)||\bbW^{\mu\nu}_{k,l}-\tilde\bbW^{\mu\nu}_{k,l}|\\
				&\le |R_{k,l}-\tilde R_{k,l}|+\|\Delta\bx\|_{\infty;[0,l]}\|\bbW\|_{p/2;[k,l]}+E_{k,l}^{2/p}.
			\end{align*}
			Using \Cref{lem:J.bound,lem:J.bound.diff,eq:R.diff.bound} we see that the right-hand side is bounded by
			\begin{equation*}
				\begin{multlined}
					C_{p,N}\Bigl\{\bigl(\|\Delta I\|_{p/2;[k,l]}+\|\Delta\bx\|_{\infty;[0,l]}(\|\tilde
						I\|_{p/2;[k,l]}+\|\tilde\bx\|^2_{p;[k,l]})+\|\Delta\bx\|_{p;[k,l]}\omega_{k,l}^{1/p}\bigr)\|\bw\|_{p;[k,l]}\\
					+(\|\tilde
					I\|_{p/2;[k,l]}+\|\tilde\bx\|^2_{p;[k,l]})\varepsilon^{1/p}_{k,l}+(\|\Delta\bx\|_{p;[k,l]}+\|\Delta\bx\|_{\infty;[0,l]}\|\tilde\bx\|_{p;[k,l]})\|\bbW\|_{p/2;[k,l]}\\
				+\|\tilde\bx\|_{p;[k,l]}E^{2/p}_{k,l}\Bigr\}+\|\Delta\bx\|_{\infty;[0,l]}\|\bbW\|_{p/2;[k,l]}+E^{2/p}_{k,l}.
				\end{multlined}
			\end{equation*}
			Defining the control
			\[
				Q_{k,l}\coloneq \begin{multlined}[t] 24^{p/2-1}C^{p/2}_p\Bigl\{\bigl(\|\Delta I\|^{p/2}_{p/2;[k,l]}+\|\Delta\bx\|^{p/2}_{\infty;[0,l]}(\|\tilde
					I\|^{p/2}_{p/2;[k,l]}+\|\tilde\bx\|^p_{p;[k,l]})+\|\Delta\bx\|^{p/2}_{p;[k,l]}\omega_{k,l}^{1/2}\bigr)\|\bw\|^{p/2}_{p;[k,l]}\\
					+(\|\tilde
					I\|^{p/2}_{p/2;[k,l]}+\|\tilde\bx\|^p_{p;[k,l]})\varepsilon^{1/2}_{k,l}+(\|\Delta\bx\|^p_{p;[k,l]}+\|\Delta\bx\|^p_{\infty;[0,l]}\|\tilde\bx\|^p_{p;[k,l]})^{1/2}\|\bbW\|^{p/2}_{p/2;[k,l]}\\
				+\|\tilde\bx\|^{p/2}_{p;[k,l]}E_{k,l}\Bigr\}
				\end{multlined}
			\]
			we obtain the bound
			\[
				|\Delta I_{k,l}|\le 3^{1-2/p}\left( Q_{k,l}+\|\Delta\bx\|_{\infty;[0,l]}^{p/2}\|\bbW\|_{p/2;[k,l]}^{p/2}+E_{k,l} \right)^{2/p}.
			\]
			By \Cref{lem:pvar.control} the same bound holds for \(\|\Delta I\|_{p/2;[k,l]}\).
			Now, taking \(k<l\) close enough such that
			\[
				\t|\bW|_{p;[k,l]}\le\frac{1}{2\times 72^{1-2/p}\times C_{p,N}}
			\]
			we see that
			\[
				\|\Delta I\|_{p/2;[k,l]}\le 2\times 72^{1-2/p}C_{p,N}\left\{ \|\Delta\bx\|_{\infty;[0,l]}\tilde
					U^{2/p}_{k,l}+\|\Delta\bx\|_{p;[k,l]}\omega_{k,l}^{1/p}+\tilde
				U_{k,l}^{2/p}\varepsilon_{k,l}^{1/p}+V^{1/p}_{k,l}\|\bbW\|_{p/2;[k,l]}^{1/2}+\|\tilde\bx\|_{p;[k,l]}E^{2/p}_{k,l}\right\}
			\]
			where now
			\begin{align*}
				\tilde U_{k,l}&\coloneq\|\tilde I\|_{p/2;[k,l]}^{p/2}+\|\tilde\bx\|_{p;[k,l]}^p\\
				V_{k,l}&\coloneq \|\Delta\bx\|_{p;[k,l]}^p+\|\Delta\bx\|^p_{\infty;[0,l]}\|\tilde\bx\|_{p;[k,l]}^p
			\end{align*}
			are new controls as well.

	\item We now use the fact that 
		\[
			\bx_{k,l}=I_{k,l}+\sum_{\mu=1}^df_\mu(\bx_k)\bw^\mu_{k,l}
		\]
		to obtain that
		\begin{align*}
			|\Delta\bx_{k,l}|&\le|\Delta I_{k,l}|+|\Delta\bx_k|\sum_{\mu=1}^d|\tilde\bw^\mu_{k,l}|+\sum_{\mu=1}^d|\Delta\bw^\mu_{k,l}|\\
			&\le \|\Delta I\|_{p/2;[k,l]}+\|\Delta\bx\|_{\infty;[0,l]}\|\bw\|_{p;[k,l]}+\varepsilon_{k,l}^{1/p}.
		\end{align*}
		From the previous bound on \(\|\Delta I\|_{p/2;[k,l]}\) we get that
		\[
			|\Delta\bx_{k,l}|\le\begin{multlined}[t]2\times 72^{1-2/p}C_{p,N}\Bigl\{\|\Delta\bx\|_{\infty;[0,l]}\tilde
					U^{2/p}_{k,l}+\|\Delta\bx\|_{p;[k,l]}\omega_{k,l}^{1/p}+\tilde U_{k,l}^{2/p}\varepsilon_{k,l}^{1/p}\\
				+V^{1/p}_{k,l}\|\bbW\|_{p/2;[k,l]}^{1/2}+\|\tilde\bx\|_{p;[k,l]}E^{2/p}_{k,l}+\|\Delta\bx\|_{\infty;[0,l]}\|\bw\|_{p;[k,l]}+\varepsilon_{k,l}^{1/p}\Bigr\}
			\end{multlined}
		\]
		Taking \(p\)-variation we see that
		\[
			\|\Delta\bx\|_{p;[k,l]}\le L_p\begin{multlined}[t]\Bigl\{\|\Delta\bx\|_{\infty;[0,l]}\tilde
					U^{2/p}_{k,l}+\|\Delta\bx\|_{p;[k,l]}\omega_{k,l}^{1/p}+\tilde U_{k,l}^{2/p}\varepsilon_{k,l}^{1/p}+V^{1/p}_{k,l}\|\bbW\|_{p/2;[k,l]}^{1/2}\\
				+\|\tilde\bx\|_{p;[k,l]}E^{2/p}_{k,l}+\|\Delta\bx\|_{\infty;[0,l]}\|\bw\|_{p;[k,l]}+\varepsilon_{k,l}^{1/p}\Bigr\}
			\end{multlined}
		\]
		with \(L_p\coloneq 2\times 7^{1-1/p}\times 24^{1-2/p}\times C_{p,N}\).
		Using again the fact that \(k<l\) are chosen so that \(\t|\bW|_{p;[k,l]}\le L_p^{-1}<1\) , and the a priori estimate in \Cref{thm:rough.apriori} we obtain that
		\[
			\|\Delta\bx\|_{p;[k,l]}\le L_p\|\Delta\bx\|_{\infty;[0,l]}(\tilde U_{k,l}^{2/p}+\t|\bW|_{p;[k,l]}+\t|\tilde\bW|_{p;[k,l]})+L'_p(\varepsilon_{k,l}^2+E_{k,l})^{2/p}.
		\]
		Now, we notice that
		\[
			\tilde U_{k,l}=\|\bx\|_{p;[k,l]}^p+\|\tilde I\|_{p/2;[k,l]}^{p/2}\le (K_p^p+(K_p')^{p/2})\t|\tilde\bW|_{p;[k,l]}^p
		\]
		so that
		\[
			\tilde U_{k,l}^{2/p}\le L_p^{-1}(K_p^2+K_p')\t|\tilde\bW|_{p;[k,l]}.
		\]

		Trivial estimates show that the same bound holds when \(l=k+1\), so by the rough Grönwall lemma and an argument similar to the Young case we obtain
		\[
			\|\Delta x\|_{\infty;[0,N]}\le c'_p e^{c_{p,N}(\t|\bW|_{p;[0,N]}^p+\t|\tilde\bW|_{p;[0,N]}^p)}\left( |\bx_0-\tilde\bx_0|+\rho_p(\bW,\tilde\bW) \right).
		\]
		where
		\[
			c_{p,N}\coloneq 2^pe^{2p}(L_p+K_p^2+K_p')^p,\quad c_{p,N}'=2^{1-2/p}c_{p,N}.\qedhere
		\]
	\end{enumerate}
\end{proof}

\appendix
\section{Comparison of architectures}
\label{sec:comp-arch}
In this section we show that the residual architectures mentioned in the introduction, namely
\begin{align}
	\by_{k+1}&= \by_k+\sigma(\by_k,\theta_k),\label{eq:arch.1}\\
	\bx_{k+1}&= \bx_k+\sum_{\mu=1}^df_\mu(\bx)\bw^\mu_{k,k+1}\label{eq:arch.2}
\end{align}
are related to each other, in the sense that any evolution represented by \cref{eq:arch.1} may be obtained as a projection of the evolution governed
by \eqref{eq:arch.2}.
Note that we allow the dimensions of the noises, as well as the vector fields, to be different in each architecture.
Consider the maps
\[\mathfrak Y\colon\R^m\times C(\R^m\times\mathcal M_{m\times m},\R^m)\times\mathcal M_{m\times m}^N\to\R^m,\quad\mathfrak X\colon\R^n\times C(\R^n,\R^n)^d\times(\R^d)^N\to\R^n\]
defined by
\[
	\mathfrak Y(\by,\sigma,\theta)\coloneq \by_N,\quad\mathfrak X(\bx,f,\bw)=\bx_N
\]
where \(\by_k\) and \(\bx_k\) solve \cref{eq:arch.1,eq:arch.2} respectively, with \(\by_0=\by\), \(\bx_0=\bx\).
The following result draws on ideas by Kidger, Morrill, Foster and Lyons \cite{KMF+2020}.

\begin{proposition}
	Fix \(d=m^2+1\) and \(n=m+d\), and set \(\pi\colon\R^n\to\R^m\) be the projection onto the first \(m\) coordinates. Then the inclusion
	\[
		\mathfrak Y(\R^m\times C(\R^m\times\mathcal M_{m\times m},\R^m)\times\mathcal M_{m\times m}^N)\subset\pi\circ\mathfrak X(\R^n\times C(\R^n,\R^n)^d\times(\R^d)^N)
	\]
	holds.
\end{proposition}
The content of this is result is that we may emulate the non-linear evolution of \cref{eq:arch.1} by a linear control system of greater dimension as
in \cref{eq:arch.2}.
Since \(\pi\) is a Lipschitz map, this has no repercussion for our estimates.
Therefore, it suffices to study systems linear in the control.
\begin{proof}
	Let \(\bw\) be the \((m^2+1)\)-dimensional noise obtained by flattening of the \(\theta\) matrices and adding a time component, i.e., for
	\(\mu\in\{1,\dotsc,m^2\}\) set
	\[
		\bw^\mu_k=\sum_{j=0}^{k}\theta_j(\lfloor\mu/m\rfloor,\mu\mod m)
	\]
	and \(\bw^{d}_k=k\).
	Let \(\tilde\pi\colon\R^{n}\to\R^d\) denote the projection onto the last \(d\) coordinates and let \(e_1,\dotsc,e_{n}\) denote the standard basis of
	\(\R^{n}\).	
	Define the vector fields \(f_i\colon\R^{n}\to\R^{n}\) by
	\begin{align*}
		f_\mu(\bx)&\coloneq e_{m+\mu},&\mu&=1,\dotsc,d-1\\
		f_{d}(\bx)&\coloneq e_{m+d}+\sum_{\nu=1}^m\sigma_\nu(\pi(\bx),\tilde\pi(\bx))e_\nu.
	\end{align*}
	Therefore, the corresponding solution to \cref{eq:arch.2} satisfies
	\begin{align*}
		\bx_{k+1}&= \bx_k+\sum_{\mu=1}^df_\mu(\bx_k)(\bw_{k+1}^\mu-\bw_k^\mu)\\
		&= \bx_k+\sum_{\nu=1}^m\sigma_\nu(\pi(\bx_k),\tilde\pi(\bx_k))e_\nu+\sum_{\mu=1}^d(\bw_{k+1}^\mu-\bw_k^\mu)e_{m+\mu}
	\end{align*}
	In particular
	\[
		\tilde\pi(\bx_{k+1})=\tilde\pi(\bx_k)+\theta_{k+1}-\theta_k,
	\]
	that is, \(\tilde\pi(\bx_k)=\theta_k\).
	Therefore,
	\[
		\pi(\bx_{k+1})=\pi(\bx_k)+\sigma(\pi(\bx_k),\theta_k)
	\]
	so that, if we pick an initial condition \(\bx\in\R^n\) such that \(\pi(\bx)=\by\in\R^m\) we immediately see that
	\(\bx_N=\by_N\).
\end{proof}
\bibliographystyle{arxiv2}
\bibliography{deeplearn}
\end{document}